\documentclass[10pt]{article}
\usepackage[a4paper,margin=1in]{geometry}
\usepackage{lipsum}

\usepackage{kotex}
\usepackage{xcolor,colortbl}
\usepackage{xspace}
\usepackage{adjustbox}
\usepackage{multirow}
\usepackage{graphicx}
\usepackage{amsmath}
\usepackage{subcaption}
\usepackage{amsthm} 
\usepackage{placeins}
\theoremstyle{definition}
\newtheorem{definition}{Definition}
\newtheorem{proposition}{Proposition} 

\definecolor{tabfirst}{rgb}{1, 0.7, 0.7}
\definecolor{tabsecond}{rgb}{1, 0.85, 0.7}
\definecolor{tabthird}{rgb}{1, 1, 0.7}

\usepackage{pifont}
\usepackage[utf8]{inputenc} 
\usepackage[T1]{fontenc}    
\usepackage{hyperref}       
\usepackage{url}            
\usepackage{booktabs}       
\usepackage{amsfonts}       
\usepackage{nicefrac}       
\usepackage{microtype}      
\usepackage{xcolor}         
\usepackage{lipsum}
\usepackage{wrapfig, booktabs}
\usepackage{booktabs}    
\usepackage{subcaption}  
\usepackage{tabularx}
\usepackage{booktabs} 
\usepackage{amsmath, amsthm, amssymb}
\newtheorem{remark}{Remark}

\usepackage{natbib}
\usepackage{caption}

\title{Error-Conditioned Neural Solvers}

\author{%
\normalsize
  Haina Jiang\textsuperscript{*1} \quad
  Liam Wang\textsuperscript{*1} \quad
  Peng-Chen Chen\textsuperscript{1} \quad
  Min Seop Kwak\textsuperscript{1,2} \quad\\[1.2ex]
    \normalsize
  Seungryong Kim\textsuperscript{2} \quad
  Brian Bell\textsuperscript{3} \quad
  Jeong Joon Park\textsuperscript{1,\textdagger} \\[1.2ex]
  \normalsize
  \textsuperscript{1}University of Michigan \quad
  \textsuperscript{2}KAIST AI\quad
  \textsuperscript{3}Los Alamos National Laboratory\\[0.5ex]
    \footnotesize\textsuperscript{*}Equal contribution. \quad \textsuperscript{\textdagger}Corresponding author.%
}
\begin{document}

\date{}

\maketitle

\begin{abstract} Neural surrogate models offer fast approximate mappings from PDE parameters to solutions, but they typically treat solving as a purely statistical task: once trained, they struggle to correct their own constraint violations and extrapolate beyond the training distribution. Recent hybrid methods promote physical correctness by targeting the PDE residual via gradient descent or Gauss--Newton steps, but inherit the compute cost and instability of the underlying classical optimizers. We show, theoretically and empirically, that numerically minimizing the PDE residual can be an unreliable proxy for reconstruction accuracy in ill-conditioned systems, explaining why these methods often do not make accurate predictions despite achieving low residuals. We propose error-conditioned Neural Solvers (ENS), built on a different principle: rather than an optimization target, the PDE residual field is passed as a direct input to the network at each iteration, enabling it to read the spatial structure of its own errors and learn an update policy to iteratively correct its predictions. Across four PDE families, ENS attains the highest prediction accuracy in the large majority of settings, with gains reaching $10\times$ on turbulent Kolmogorov flow, while avoiding the expensive compute cost of hybrid methods. ENS's learned correction policy generalizes under distribution shift, including zero-shot parameter changes and cross-equation transfer, where its relative advantage is largest in the ill-conditioned regimes where residual minimization is least reliable. Project website: \href{https://neuralsolver.github.io/}{neuralsolver.github.io}\end{abstract}

\section{Introduction}

Neural operators dramatically accelerate PDE solving by replacing expensive numerical solvers with fast, learned mappings from parameters to solution fields~\citep{lu2021learning, li2020fourier}. Yet they typically treat solving as a purely statistical regression problem, which introduces two fundamental limitations: they typically have no mechanism to assess or correct constraint violations in their own outputs and struggle to extrapolate reliably beyond the training distribution. Traditional numerical methods~\citep{aliabadi2020boundary, solin2005partial} avoid these by construction, but their cost scales rapidly with resolution and dimensionality, making them the primary bottleneck in large-scale simulation applications.

Recent hybrid methods incorporate the PDE residual at test time to impose physics-based corrections. These methods share a common structure: the residual is used as an optimization target via gradient descent in PINO~\citep{li2024physics}, or constraint-manifold projection in DiffusionPDE~\citep{huang2024diffusionpde} and PCFM~\citep{utkarsh2025physics}, relying on expensive numerical procedures to translate residual information into solution corrections. Moreover, we show theoretically and empirically that this numerical optimization can expose these hybrid methods to a residual-reconstruction gap: minimizing the PDE residual is often an unreliable proxy for solution accuracy in ill-conditioned systems (Prop.~\ref{prop:residual_recon}), explaining why they tend to achieve low prediction accuracy despite achieving low residuals.

We present error-conditioned \textit{Neural Solvers} (ENS), a new framework built on a different principle: the PDE residual field should be a direct input to the network, instead of an optimization target. At each iteration, ENS receives the current solution prediction alongside its PDE residual field, and learns to produce a correction by directly reading the spatial structure of the error. ENS is trained under reconstruction supervision alone, where the residual is an input signal rather than an objective to minimize, directly sidestepping the residual-reconstruction gap (Prop.~\ref{prop:residual_recon}) that undermines hybrid methods. Applied recurrently, ENS progressively refines the solution without explicit numerical optimization, with computation cost cheaper than existing hybrid methods, and with robustness to initialization that enables generalization to unseen equations.

We evaluate ENS across four PDE families and increasingly challenging regimes—in-distribution prediction, super-resolution, coefficient extrapolation, and cross-equation transfer. ENS attains the most accurate predictions in most of the tested settings, by up to an order of magnitude on the most ill-conditioned problems, and achieves low PDE residuals together with low reconstruction error (Fig.~\ref{fig:motivation_consistency}), despite never directly minimizing the residual. Under distribution shift, ENS's advantage is largest in the ill-conditioned regimes where minimizing the residual is least reliable: data-driven operators degrade sharply, while residual-minimizing methods remain competitive in benign regimes but are slow and lose their advantage as conditioning worsens. ENS further exhibits initialization-robust convergence: trajectories spanning seven orders of magnitude in initial residual converge to the same residual floor (Fig.~\ref{fig:motivation_convergence}), a reliability property consistent with its strong extrapolation in ill-conditioned settings.

Our contributions are: (i) ENS, a recurrent residual-conditioned \textit{Neural Solver} achieving up to an order-of-magnitude improvement in reconstruction accuracy across diverse extrapolation settings; (ii) a characterization of ENS, including initialization robustness, ablation experiments, and its design choices; and (iii) a theoretical and empirical analysis of the residual-reconstruction gap (Prop.~\ref{prop:residual_recon}), explaining why existing hybrid methods often struggle to recover accurate solutions despite achieving low residuals.

\begin{figure*}[t]
\centering
\includegraphics[width=\linewidth]{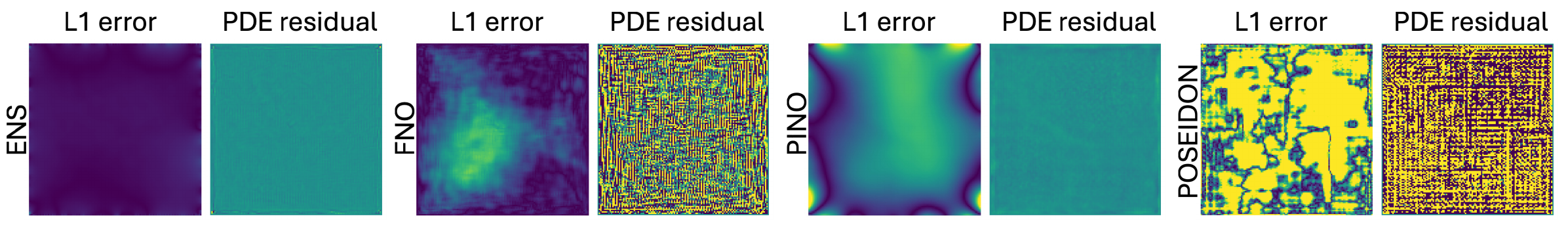}
\caption{
\textbf{Prediction and Physical accuracy of PDE solvers.} For each method we show the reconstruction
loss ($L_1$  error) and PDE residual field $r(u) = \mathcal{F}(u;f)$ computed on the same test instance.
Feed-forward solvers (FNO, POSEIDON) exhibit large residual fields, indicating
systematic violation of the governing equations. While PINO reduces the PDE error, it does not translate to accurate solution fields (following Prop.~\ref{prop:residual_recon}). Our ENS produces both near-perfect reconstruction and physical accuracy. Corresponding fields are visualized on the same color scale.
}
\label{fig:motivation_consistency}
\end{figure*}
\section{Related Work}
\label{sec:background}
Neural operators learn mappings between function spaces for fast prediction across PDE families \citep{li2020fourier, lu2021learning, kovachki2023neural}, with foundation models extending this to large-scale multi-PDE pretraining \citep{herde2024poseidon}. These feed-forward models are \emph{physics-blind} at inference: they receive no feedback about their own prediction error and can incur large constraint violations even when outputs appear plausible. PINNs \citep{raissi2019physics} instead impose physics on a single instance but require per-instance retraining (Fig.~\ref{fig:motivation_convergence}).

To add inference-time physics feedback, hybrid methods use the PDE residual as a correction signal. Optimization-based methods reach a low residual through a inference-time procedure---parameter fine-tuning \citep{li2024physics}, residual-guided diffusion \citep{huang2024diffusionpde, jacobsen2025cocogen}, constraint projection \citep{christopher2024constrained, utkarsh2025physics}, or zero-shot constrained sampling \citep{cheng2025gradient}---while analytical methods apply a closed-form accelerated linearized correction at the prediction \citep{cao2023residual, jha2024residual, huang2026physicscorrect}. In every case the residual is an \emph{optimization target}, not an \emph{input} to the network; we analyze the resulting cost and reliability in Sec.~\ref{sec:existing} and Tab.~\ref{tab:comparison}. Iterative learned refinement has also been studied without residual conditioning, using noise levels \citep{lippe2023pde} or corrections embedded in a coarse solver \citep{wei2026inc}. 

Concurrently to our work, PRISMA \citep{sawhney2025beyond} moves residual information inside a diffusion neural operator via spectral attention, sharing our motivation. It differs in both mechanism and problem: PRISMA's residual is computed from noisy observations and conditions a generative denoiser for inverse problems, whereas ENS computes $r^{(k)}$ from the current prediction and feeds it to a corrector that produces an explicit update, closing the feedback loop of Def.~\ref{def: neural_solver} for forward solving under distribution shift.

Appendix Sec.~\ref{sec: rel_appendix} provides detailed descriptions and covers generative and unrolled-network methods.
\section{The Unified View of Hybrid Residual Correction Methods}
\label{sec:existing}
The hybrid methods in Sec.~\ref{sec:background} use the PDE residual as an error signal at inference time. Although they differ in detail, they share a common structure: the residual is used either as a target for gradient-based updates or as the input to a linear solve. Tab.~\ref{tab:comparison} groups them by the classical numerical scheme each one corresponds to. We show that this shared structure can lead to two failure modes: minimizing the residual can be an unreliable proxy for reconstruction accuracy, and the second-order updates they rely on are stable only close to the solution.

\begin{table*}[h!]
\centering
\small
\caption{\textbf{Taxonomy of PDE solving methods by how the residual enters the network.} 
Existing methods use the residual as a training signal, external gradient, external 
projection, or external linear solve. ENS is the first \textit{Neural Solver} to receive 
the residual as a direct input. Inference costs are measured on 20-frame Navier-Stokes scenes. Init.\ sensitivity indicates whether convergence requires 
initialization close to $u^\ast$.}
\label{tab:comparison}
\begin{tabular}{lcccc}
\toprule
\textbf{Method} & \textbf{Baselines} & \textbf{Residual Usage} & \textbf{Inference Cost (s.)} 
& \textbf{Init.\ Sensitive} \\
\midrule
Feed-forward & FNO, POSEIDON & None (training only) & $0.0086, 3.5$ & \ding{55} \\
GD-PINNs & PINO & External gradient & $44$ & \ding{51} \\
GD optim. & DiffusionPDE & External gradient & $191$ & \ding{51} \\
GN projection & PCFM & External projection & $251$ & \ding{51} \\
Newton & PhysicsCorrect  & External linear solve & $ 0.13 $ & \ding{51} \\
\midrule
ENS (Ours) & --- & Network input & $0.19$ & \ding{55} \\
\bottomrule
\end{tabular}
\end{table*}

\paragraph{First-order.} PINO~\citep{li2024physics} performs PINNs-style gradient descent through backpropagation on the residual loss:
\begin{equation}
\hat{\theta} \leftarrow \hat{\theta} - \alpha\nabla_\theta\|r(\hat{u})\|^2.
\end{equation}
DiffusionPDE~\citep{huang2024diffusionpde} applies the first-order update as diffusion guidance:
\begin{equation}
\hat{\mathbf{x}}_{t-1} = \mathrm{denoise}(\mathbf{x}_t) - \gamma\nabla_{\mathbf{x}_t}\|r(\hat{u}_0)\|^2.
\end{equation}
In both cases the correction is applied externally, with the residual used as an optimization target.\\
PhysicsCorrect~\citep{huang2026physicscorrect} corrects predictions by locally
linearizing the PDE residual and inverting its Jacobian. It may deteriorate
when the linearization is inaccurate or the Jacobian is ill-conditioned
(Fig.~\ref{fig: PhyCorr}).

\paragraph{Second-order} PCFM~\citep{utkarsh2025physics} applies a penalized Gauss-Newton correction at each flow step, balancing proximity to the current estimate $\hat{u}_{\tau'}$ with soft constraint satisfaction evaluated at the flow-extrapolated point $u + \gamma v_\theta(u, \tau')$, where $\gamma = 1 - \tau'$:
\begin{equation}
u_{\tau'} = \arg\min_{u}\,\|u - \hat{u}_{\tau'}\|^2 + \lambda\left\|h\!\left(u + \gamma v_\theta(u, \tau')\right)\right\|^2.
\end{equation}
PCFM's scheme inherits Gauss-Newton's initialization sensitivity and, for full PDE residuals, per-step cost grows to $O(n^2)$--$O(n^3)$---equivalent to or worse than classical solvers; PCFM's claimed $O(n)$ cost holds only for low-dimensional constraints ($m \ll n$) and does not extend to full PDE residuals. Moreover, the regularization term $\|u - \hat{u}_{\tau'}\|^2$ is symptomatic of optimization pitfalls below shared by the hybrid methods.

\begin{proposition}[Residual-Reconstruction Gap]
\label{prop:residual_recon}
For a PDE with true solution $u^\star$ satisfying $r(u^\star) \approx 0$, minimizing $\|r(u)\|_2$ does not reliably reduce reconstruction error. Specifically, let $J_r(u^\star)$ denote the Jacobian of $r$ at the true solution $u^\star$, with smallest singular value $\sigma_{\min}(J_r(u^\star))$. For any $\varepsilon > 0$ there exists $\tilde{u}$ with $\|r(\tilde{u})\|_2 \leq \varepsilon$ yet:
\begin{equation}
\|\tilde{u} - u^\star\|_2 \;\geq\; \frac{\varepsilon}{\sigma_{\min}(J_r(u^\star))}.
\end{equation} 
When $\sigma_{\min}$ is small---as in high-wavenumber Helmholtz or low-viscosity Navier-Stokes---residual minimization becomes an arbitrarily unreliable proxy for reconstruction accuracy. This follows from classical conditioning theory; we apply it to explain why existing hybrid methods achieve low residuals yet poor reconstruction for ill-conditioned systems.
\end{proposition}

Geometrically, the zero-residual manifold is nearly flat along the smallest singular direction of $J_r(u^\star)$, permitting large deviations from $u^\star$ at negligible residual cost. This analysis explains why the low PDE-residuals of hybrid methods do not always lead to superior reconstruction accuracy (Tab.~\ref{tab:forward_NS}).

\begin{remark}[Initialization Sensitivity of Linearization Methods]
\label{rem:linearization}
Let $M$ bound the curvature $\|D^2 r\|_{\mathrm{op}}$ near $u^\star$. Newton and Gauss--Newton corrections are contractive only within a basin of radius $O(\sigma_{\min}(J_r(u^\star))/M)$ around the true solution $u^\star$~\citep{deuflhard2011newton}.
\end{remark}
Linearization-based hybrid methods can therefore grow unreliable outside the basin (see Fig.~\ref{fig:motivation_convergence}a, Tab.~\ref{tab:Darcy}), e.g., when distribution shift pushes the initialization far from $u^\star$. ENS sidesteps this: its learned corrector $\mathcal{C}_\phi$ replaces local linearization with a nonlinear map over the full residual field, and training on reconstruction loss avoids the objective-level failure of Prop.~\ref{prop:residual_recon}.
\begin{figure*}[t]
    \centering
    \begin{minipage}[t]{0.725\linewidth}
        \includegraphics[width=1.00\linewidth]{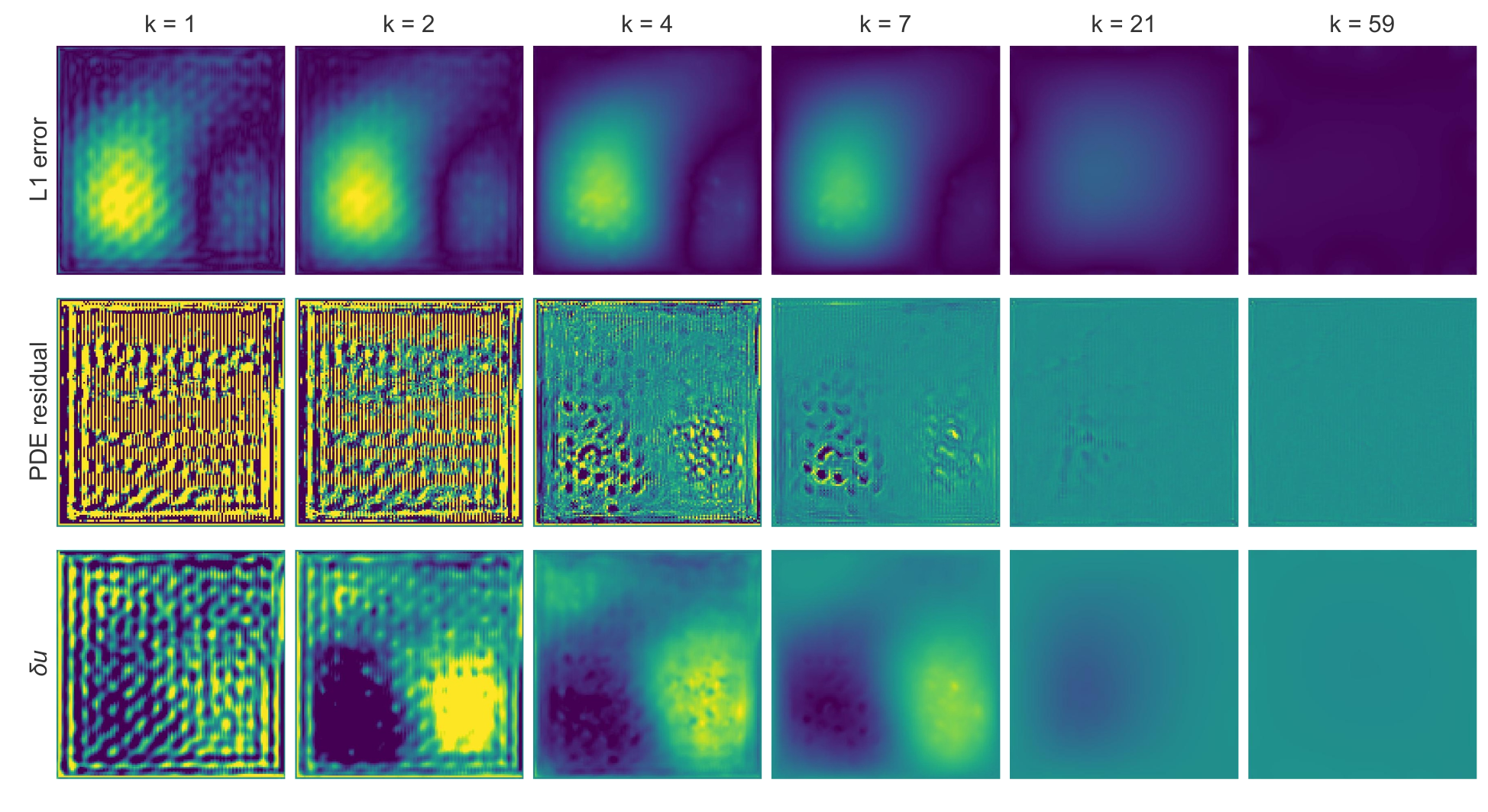}
\caption{
\textbf{Evolution of reconstruction loss, PDE residual $r(\hat{u}^{(k)})$, \\and learned correction $\delta\hat{u}^{(k)}$ across iterations}. Note that the correction $\delta\hat{u}^{(k)}=\mathcal{C}_\phi(r(\hat{u}^{(k)}),\cdot)$ reflects the spatial structure of $r(\hat{u}^{(k)})$ and drives down the PDE residual as well as the reconstruction prediction error.
}
\label{fig:progression}
    \end{minipage}
    \hfill 
    \begin{minipage}[t]{0.26\linewidth}
        \centering
        \includegraphics[width=\linewidth]{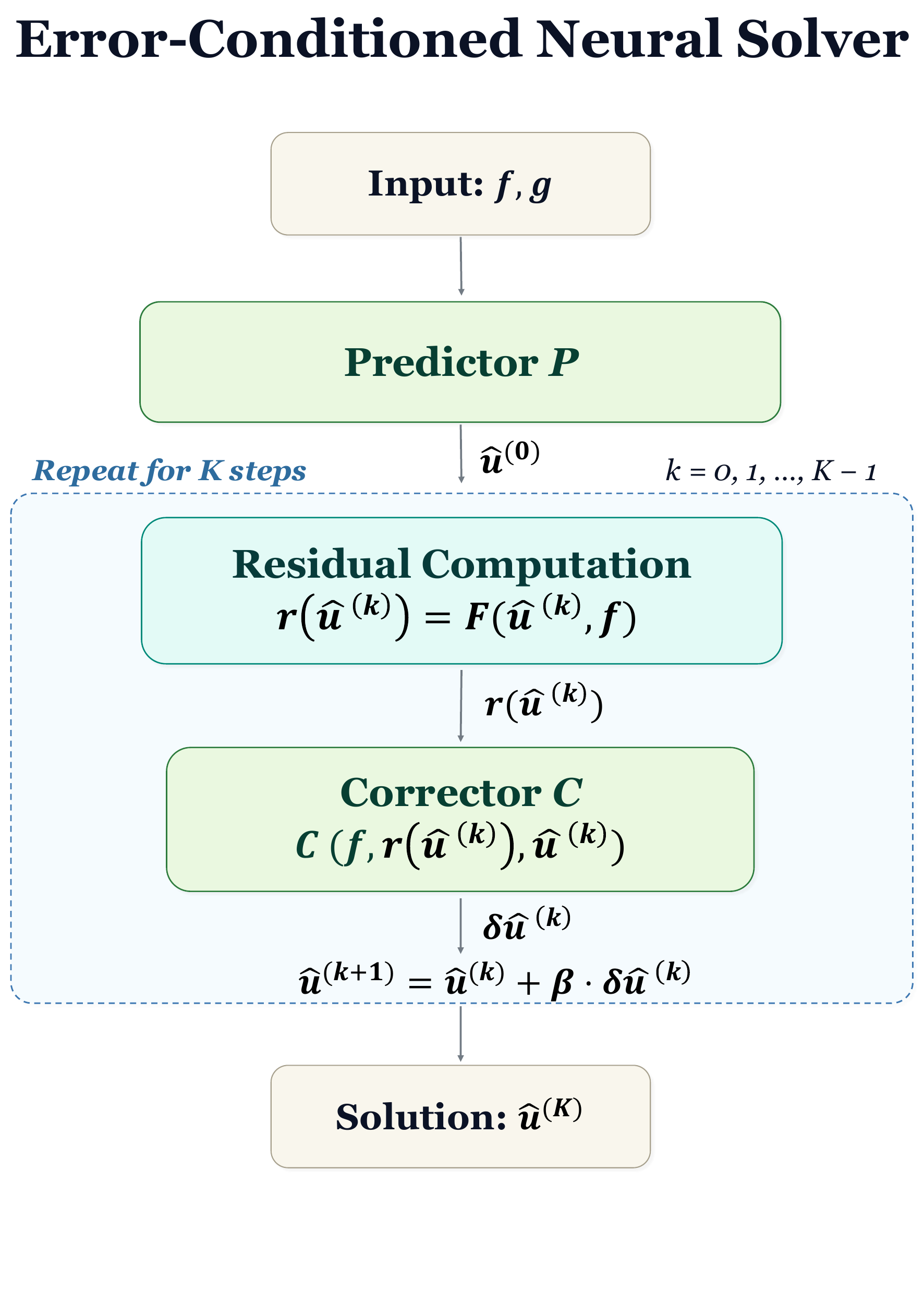} 
        \caption{\textbf{ENS inference-time schematic diagram}.}
        \label{fig:right}
    \end{minipage}
\end{figure*}

\begin{figure}[t]
        \centering
        \includegraphics[width=0.85\linewidth]{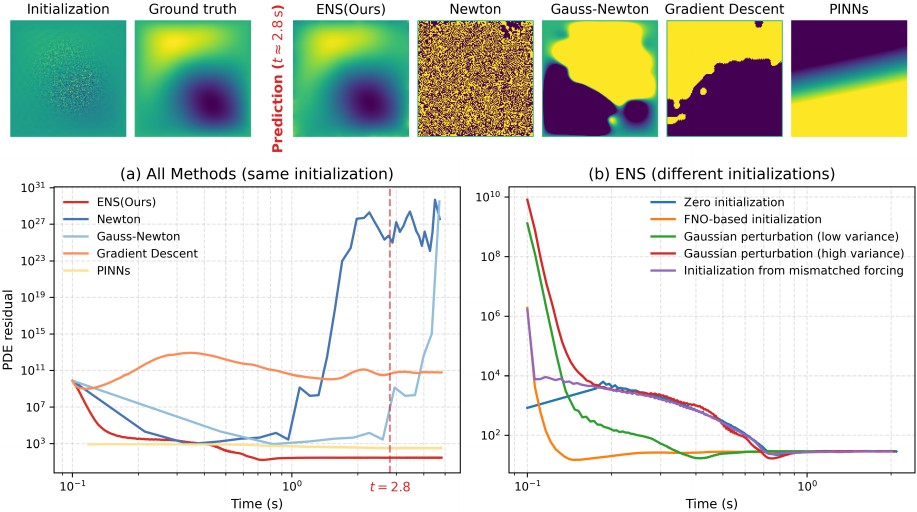} 
\caption{\textbf{PDE residual vs.\ compute time on the nonlinear Helmholtz equation.} \emph{Top right:} predictions at $t=2.8$\,s. \emph{Bottom (a):} residual trajectories; all methods except PINNs share the noisy initialization (top-left). Pure Newton and Gauss-Newton diverge ($>10^{20}$) from Jacobian unreliability far from the solution; PINNs converge very slowly; gradient descent stagnates near its initialization. Recent hybrid methods (DiffusionPDE, PINO-TTOP, PCFM) inherit these limitations, motivating ENS's learned correction scheme. ENS reduces the residual by over an order of magnitude within the same budget and converges to a common floor from initial residuals spanning 7 orders of magnitude (\emph{Bottom (b)}).}       \vspace{-0.3cm}
        \label{fig:motivation_convergence}
\end{figure}

\section{Error-Conditioned Neural Solvers}
\label{sec:rns}

The methods surveyed above share a common structure: whether via gradient descent or analytical linearization, hybrid methods apply residual as an optimization target rather than read directly as a dynamic error signal, while feed-forward neural operators do not see the residual at all. We propose \textbf{Neural Solvers}, a new class of methods with a simple principle: \textit{the network should read and correct its own errors,} akin to the iterative corrections of classical solvers.
\begin{definition}[Neural Solver]\label{def: neural_solver}
A \textbf{Neural Solver} consists of an initialization $\hat{u}^{(0)}$ and a learned corrector $C_\phi$ that iteratively refines the solution $\hat{u}^{(k)}$ by reading its explicit 
error signal $e^{(k)}$:
\begin{equation}
    \hat{u}^{(k+1)} = \hat{u}^{(k)} + \beta \cdot C_\phi\!\left(\hat{u}^{(k)},\, e^{(k)};\, f, g\right)
    \label{eq:neural_solver}, \qquad \text{for } k = 0, \ldots, K-1
\end{equation}
where $e^{(k)}$ encodes the current prediction's failure to satisfy a 
constraint, and $\beta > 0$ is a step size.
\end{definition}

\textbf{Neural Operators} correspond to the special case where $K=0$ in Def.~\ref{def: neural_solver}---they produce a solution estimate $\hat{u}^{(0)} = \mathcal{N}_\theta(f, g)$ in a feed-forward pass, with no corrector applied. Without error feedback, the network cannot assess or correct its own prediction at inference time, regardless of its depth or capacity.

The choice of error signal $e^{(k)}$ defines a family of Neural Solvers---for instance, the PDE residual field for physics constraint violation, observation mismatch for inverse problems, or symmetry residuals for equivariance constraints. In this work, we instantiate $e^{(k)}$ as the PDE residual field $r^{(k)} = \mathcal{F}(\hat{u}^{(k)}; f)$, the most direct available measure of physical constraint violation.
We also refer this instantiation as \textbf{ENS} (Error-Conditioned Neural Solver). A predictor network $\mathcal{P}_\theta$ produces an initial estimate, and a learned corrector $\mathcal{C}_\phi$ iteratively refines it by reading the residual field:
\begin{equation}
\hat{u}^{(0)} = \mathcal{P}_\theta(f, g), \qquad \delta\hat{u}^{(k)} = \mathcal{C}_\phi\!\left(\hat{u}^{(k)},\, r^{(k)};\, f,\, g\right), \qquad \hat{u}^{(k+1)} = \hat{u}^{(k)} + \beta\,\delta\hat{u}^{(k)}
\label{eq:rns_update}
\end{equation}
where $r^{(k)} = \mathcal{F}(\hat{u}^{(k)};f)$ is recomputed at each step. The residual field enters $\mathcal{C}_\phi$ as a spatial input channel---the network reads where and how the solution currently violates the governing equations and learns a nonlinear correction policy over it (Fig.~\ref{fig:progression}). The characterizations of ENS are as follow.

\paragraph{Learned correction over local linearization.}
Classical Newton iteration computes corrections by solving $J_r(\hat{u}^{(k)})\delta u^{(k)} = -r^{(k)}$, inverting the local Jacobian at each step. This is unreliable far from $u^\ast$ (Remark ~\ref{rem:linearization}) and prohibitively expensive for large systems. ENS replaces Jacobian inversion with $\mathcal{C}_\phi$, which learns a nonlinear correction policy across training that generalizes across initializations and equation instances without requiring Jacobian computation at inference. Unlike existing methods that treat the residual as an optimization target, ENS treats it as a reference input, sidestepping the residual reconstruction gap (Prop.~\ref{prop:residual_recon}).

\paragraph{Training and Inference.} We run the full correction loop and supervise every intermediate prediction against the ground truth (Fig.~\ref{fig:right}): \begin{equation} \mathcal{L} = \frac{1}{K}\sum_{k=1}^{K}\frac{1}{n}\left\|\hat{u}^{(k)} - u_{\mathrm{gt}}\right\|_2^2, \end{equation} where $\hat{u}^{(0)}=\mathcal{P}_\theta(f,g)$ is the predictor output and $\hat{u}^{(k)}$ for $k\ge 1$ are successive corrections; we use $K=5$ correction steps in training, giving five supervised predictions. All trainable components, the predictor $\mathcal{P}_\theta$ and the corrector $\mathcal{C}_\phi$, are optimized jointly under $\mathcal{L}$ by backpropagation; we detach the input to each correction step so that gradients do not propagate across iterations. This training scheme exposes $\mathcal{C}_\phi$ to the distribution of residuals it actually encounters at inference, including the compounding errors of imperfect intermediate corrections. The step size $\beta$ is a fixed hyperparameter during training.

At inference, $\mathcal{C}_\phi$ runs stably for more steps than the five used in training. For each task, we select a fixed number of inference steps $K$ empirically based on the convergence of residual norm $\|r(\hat{u}^{(k)})\|_2$ (Fig.~\ref{fig: Losscov}) from a few held-out validation samples, and use the same $K$ for all test samples within that task. The step size $\beta$ is reused unchanged across all settings except super-resolution and the Darcy permeability-shift setting, where we adopt a smaller $\beta$ due to the changed residual scale. In both cases, $\beta$ is selected by line search on a small held-out set. Apart from this, only the number of correction steps varies across settings. 

We instantiate ENS with a modified-FNO backbone for the static equations and a transformer-based backbone \citep{li2025videopde} for the turbulent flows. Consistent with Prop.~\ref{prop:residual_recon}, adding a PDE-residual loss term during training did not improve reconstruction accuracy in our settings, so we train on $\mathcal{L}$ alone.

\begin{remark}[Initialization Robustness]
\label{rem:initialization}
ENS drives $\|r(\hat{u}^{(k)})\|_2$ to a consistent floor independently of initialization, as verified empirically in Fig.~\ref{fig:motivation_convergence}b across initializations spanning seven orders of magnitude in initial residual. Remarkably, although $\mathcal{C}_\phi$ is strictly initialized with $\mathcal{P}_\theta$ during training, it generalizes to diverse corrupted initializations, signaling that ENS learns a robust correction policy rather than a simple statistical mapping. Note that this independence holds for the PDE residual but not the reconstruction loss: by Prop.~\ref{prop:residual_recon}, different initializations exhibit different trajectories on the low-residual manifold. The PDE residual converges while reconstruction loss diverges (Fig.~\ref{fig: Losscov}). 
\end{remark}

\section{Experiments}
\label{sec:experiments}

\paragraph{Setup.} We evaluate ENS on four PDE families---linear and nonlinear Helmholtz ($\Delta u + \kappa^2 u + \lambda u^3 = f$), Darcy flow ($-\nabla\cdot(a(\mathbf{r})\nabla u) = f$), Poisson ($\Delta u = af$), and Navier--Stokes in vorticity form ($\partial_\tau w + v\cdot\nabla w = \nu\Delta w + f$, $\nabla\cdot v = 0$) with diagonal-wave forcing for Navier–Stokes and monochromatic forcing for Kolmogorov flow. We train one model per equation on an in-distribution dataset and test on held-out in-distribution data plus three out-of-distribution regimes: super-resolution ($128\!\to\!256$), parameter extrapolation (wavenumber, permeability, forcing, viscosity), and cross-equation transfer (details in Tab.~\ref{tab:extrap}). We report relative $L_2$ reconstruction error and PDE-residual MSE---by Prop.~\ref{prop:residual_recon} these are not interchangeable in ill-conditioned regimes, where the residual can be driven low while the solution stays inaccurate.

\paragraph{Models.} ENS pairs a predictor $\mathcal{P}_\theta$, which produces the initial estimate, with a recurrent corrector $\mathcal{C}_\phi$, which reads the current prediction together with its PDE residual and outputs a refinement applied over $K$ steps (Sec.~\ref{sec:rns}). Both are instantiated as FNOs with CNN lifting and projection layers for the static PDEs, and as a transformer-based VideoPDE backbone~\citep{li2025videopde} for the turbulent Navier--Stokes and Kolmogorov flows. Baselines---FNO~\citep{li2020fourier}, PINO~\citep{li2024physics} with and without test-time optimization (TTOP), POSEIDON~\citep{herde2024poseidon}, DiffusionPDE~\citep{huang2024diffusionpde}, and PCFM~\citep{utkarsh2025physics}---use the authors' official implementations and default settings (PINO-TTOP to residual convergence in $500/1{,}000$ steps; DiffusionPDE's two-frame representation; PCFM's Gauss--Newton projection over $50$ flow steps); we omit methods these already compare. The static FNO and PINO baselines keep their original lifting and projection rather than ENS's CNN variant: adding the CNN does not clearly improve their prediction accuracy with mixed signals (Tab.~\ref{tab:HZ_AB}), so they are run in their original configurations. Finally, we follow PINO's implementations for finite-difference and Fourier-space residual computations. We refer to appendix Sec.~\ref{app:instantiations} of the supplementary material for details.

\paragraph{Accuracy.} ENS attains the lowest reconstruction error in the large majority of settings while keeping the residual low (Tabs.~\ref{tab:forward_helmholtz}--\ref{tab:kolmogorov}; qualitative results in Figs.~\ref{fig: StaticPDEs}--\ref{fig: DynamicPDEs}). On Helmholtz it is best on both metrics across all four regimes. On Navier-Stokes it is best in-distribution and under viscosity- and forcing-shift, regimes where feed-forward operators collapse (POSEIDON's error exceeds $0.5$ under forcing-shift), and on turbulent Kolmogorov flow, the most ill-conditioned family, its margin is largest across all settings, as predicted in Prop.~\ref{prop:residual_recon}. The baselines display the residual--reconstruction gaps that motivates reporting both metrics: residual-minimizing hybrids (PINO-TTOP, PCFM) reach low residuals but leave reconstruction errors high in ill-conditioned regimes (e.g., wavenumber and viscosity increase). Fig.~\ref{fig:jacobian} shows PINO-TTOP's $L_2$ error rising for wavenumber $\kappa=3$ even as its residual falls, while its $L_2$ error slightly decreases for wavenumber $\kappa=1$, when trained on $\kappa=2$. ENS converges for all cases. The one regime where ENS is not the most accurate is NS super-resolution, where it is outperformed by PINO-TTOP, whose spectral (Fourier) parameterization is resolution-invariant by construction, and by POSEIDON, a multiscale transformer that generalizes well across scales, neither property of which ENS's backbone has.

\paragraph{Efficiency.} ENS runs in $0.10$\,s/sample (static) and $0.19$\,s/sample (NS) in-distribution---roughly $100\times$ and $230\times$ faster than PINO-TTOP and orders of magnitude faster than DiffusionPDE and PCFM---while remaining the most accurate learned method. Out-of-distribution settings need more correction steps and thus more runtime, but ENS stays far faster than the hybrid methods and, on the expensive NS family, below the numerical solver ($0.19$\,s in-distribution vs.\ $23.3$\,s/sample). This places it on the accuracy--compute Pareto front against the learned baselines (Fig.~\ref{fig:pareto},~\ref{fig:pareto_kf}); the TTOP- and diffusion-based hybrids sit off the front, spending large compute on residuals that do not translate into accuracy.

\paragraph{Ablations.} Controlled studies isolate the source of ENS's gains; full details are in App.~\ref{app: addition}. \textit{Conditioning, not computation.} With architecture and step count fixed, replacing the residual input with a zero field stalls both losses, whereas ENS reduces both (Fig.~\ref{fig:ablation}): the gains come from reading the residual, not added iterations. \textit{Not the backbone.} On ENS's own backbone, the operator baselines and their PINO and TTOP variants still trail ENS in reconstruction error in every regime (Tabs.~\ref{tab:HZ_AB},~\ref{tab:NS_AB},~\ref{tab:KF_AB}); TTOP instead reproduces the gap of Prop.~\ref{prop:residual_recon}, driving the residual low while reconstruction stays high, mild on Navier--Stokes but severe on Kolmogorov flow, where its $L_2$ error is about $50\times$ ENS's. \textit{Input, not objective.} Conditioning on the physics-loss gradient lowers the residual but stalls reconstruction (Fig.~\ref{fig: gradient}), and a residual training term did not help. \textit{Backbone needs expressivity.} The residual input contains high-frequency information that must be representable by the backbone; on Helmholtz a pure FNO diverges while its CNN-augmented variants converge (Tab.~\ref{tab:Ablation}). Transformer backbones are also expressive enough to converge. \textit{A learned policy.} Although trained with five correction steps, ENS keeps improving past 
100 iteration at inference and stops when the residual plateaus (Fig.~\ref{fig:jacobian}, ~\ref{fig: delta_u}), indicating a learned correction rather than a fixed scheduled mapping.

\begin{table*}[htb]
  \caption{\textbf{Forward prediction results for the Helmholtz problem}. Time refers to the mean per-sample inference latency (In-distribution). Bold indicates the best result, and underlining indicates the second-best. Cross-equation refers to inferencing on non-linear Helmholtz while trained on Poisson. PINO(TTOP) and PCFM attain low PDE residuals but high $L_2$ error, showing the residual--reconstruction gap (Prop.~\ref{prop:residual_recon}).}
  \label{tab:forward_helmholtz}
  \centering
  \scriptsize
  \setlength{\tabcolsep}{3pt}
  \resizebox{\linewidth}{!}{%
  \begin{tabular}{lccccccccc}
    \toprule
    & &\multicolumn{2}{c}{In-distribution} &
    \multicolumn{2}{c}{Wavenumber-shift} & 
    \multicolumn{2}{c}{Cross-equation} &
    \multicolumn{2}{c}{Super-resolution} 
    \\
    \cmidrule(r){3-4} \cmidrule(r){5-6} \cmidrule(r){7-8} \cmidrule(r){9-10}
    Method & Time(s) &$L_2$ Rel & PDE Res & $L_2$ Rel & PDE Res & $L_2$ Rel & PDE Res& $L_2$ Rel & PDE Res\\
    \midrule
ENS (ours) & \underline{0.1} & \textbf{3.04e-03} & \textbf{3.90e-01} &  \textbf{2.59e-02} & \textbf{5.19e-01} & \textbf{8.73e-02} & 6.54e+01 & \textbf{1.82e-02} & \textbf{4.52e-01} \\
FNO & \textbf{0.0065} &  \underline{1.20e-02} & 4.21e+02 &  2.39e-01 & 3.80e+02 & 5.21e-01 &  1.33e+03 & 9.06e-02 & 1.28e+04  \\

PINO & \textbf{0.0065}& 4.19e-02 & 4.98e+01 &  2.43e-01 & 6.02e+01 & 5.27e-01  & 6.42e+03 & \underline{1.29e-01} & 2.24e+03 \\

PINO(TTOP) & 10 & 6.44e-02 & 9.65e+00 & \underline{1.61e-01} & 1.28e+01 & 4.08e-01  & \underline{3.31e+01} & 1.68e-01 & \underline{4.17e+01}  \\

POSEIDON & 0.18 & 2.70e-02 & 2.57e+04 & 2.40e-01 &  2.35e+04 & 4.80e-01 & 3.27e+04 & 5.17e-01  & 6.02e+04  \\
    DiffusionPDE & 191 & 9.85e-02 & 2.19e+04  & 2.69e-01 & 2.35e+04 & \underline{3.90e-01} & 4.21e+05 & 7.72e-01 & 2.29e+04\\
    
    PCFM & 251 & 1.08e-01 & \underline{3.91e+00} & 1.78e-01 & \underline{7.02e+00} & 6.30e-01  & \textbf{1.62e+01} & OOM & OOM  \\
    \bottomrule
  \end{tabular}
  }
\end{table*}

\vspace{-7 pt}

\begin{table*}[htb]
  \caption{\textbf{Forward prediction results for the Navier-Stokes (20 frames)}. Time refers to the mean per-sample inference latency (In-distribution). For ENS, distribution-shift requires more iterations and thus longer inference latency. DiffusionPDE only predicts the first and last frames, thus unable to compute PDE residuals. PCFM goes out of memory at $256^2$ resolution. The gap is mild here: PINO(TTOP) reaches the lowest residuals yet trails ENS in $L_2$ (Prop.~\ref{prop:residual_recon}) except super-resolution.}
  \label{tab:forward_NS}
  \centering
  \small
  \resizebox{\linewidth}{!}{
  \begin{tabular}{lccccccccc}
    \toprule
    & & \multicolumn{2}{c}{In-distribution} & \multicolumn{2}{c}{Viscosity-shift} & 
    \multicolumn{2}{c}{Forcing-shift} & 
    \multicolumn{2}{c}{Super-resolution}
    \\
    \cmidrule(r){3-4} \cmidrule(r){5-6} \cmidrule(r){7-8} \cmidrule(r){9-10}
    Method & Time(s) & $L_2$ Rel & PDE Res & $L_2$ Rel & PDE Res & $L_2$ Rel & PDE Res& $L_2$ Rel & PDE Res\\
    \midrule
ENS (ours) & 0.19 & \textbf{2.22e-03} &  \underline{4.69e-05} &  \textbf{3.05e-02} & 1.12e-02 &  \textbf{5.31e-03} & \underline{7.40e-04} & 3.98e-02  & 9.11e-02\\
FNO & \textbf{0.0086} &3.54e-02 & 8.24e-04&  1.36e-01& 1.18e-02 &5.57e-01 &3.09e-02 &7.51e-02 & 8.33e-03\\
PINO & \textbf{0.0086} & 3.57e-02 & 8.45e-04& 1.37e-01 & 1.20e-02 & 5.57e-01 & 3.10e-02 & 7.59e-02 & 8.48e-03 \\
PINO(TTOP) &43.7 & 1.17e-02 & \textbf{2.33e-05} & \underline{5.42e-02}  &  \textbf{3.79e-04} & \underline{1.45e-02} & \textbf{4.39e-05} & \underline{1.73e-02} & \textbf{7.69e-05}   \\
POSEIDON  & 3.5& \underline{4.73e-03} &  6.69e-05 &  7.45e-02 &\underline{1.17e-03} & 5.66e-01 & 2.07e-02 & \textbf{8.83e-03} & \underline{4.45e-04} \\
    DiffusionPDE & 229.22&5.59e-01  & N/A & 6.14e-01& N/A & 7.52e-01 & N/A &  6.83e-01 & N/A \\
    PCFM &192.01 & 1.34e+00&4.79e+02 & 7.18e-01& 8.33e-01& 9.32e-01 & 1.00e+04& OOM & OOM \\
    PhysicsCorrect & \underline{0.13} &  1.07e-01 & 1.78e+00 &   3.25e-01 & 4.27e+01 & 5.54e-01  &6.09e+01  & NaN & NaN \\
    \bottomrule
  \end{tabular}
  }
\end{table*}

\begin{table*}[h!]
\centering
\caption{\textbf{Forward prediction results for the Kolmogorov flow (40 frames)}. DiffusionPDE only predicts the first and last frames, thus unable to compute PDE residuals. PCFM goes out of memory in our 40-frames settings. Note the severe residual-reconstruction gap (Prop.~\ref{prop:residual_recon}) of PINO(TTOP).} 
\label{tab:kolmogorov}
\small
\resizebox{\linewidth}{!}{
\begin{tabular}{lccccccccc}
    \toprule
    & & \multicolumn{2}{c}{In-distribution} &
    \multicolumn{2}{c}{Viscosity-shift} & 
    \multicolumn{2}{c}{Forcing-shift} &
    \multicolumn{2}{c}{Super-resolution}
    \\
    \cmidrule(r){3-4} \cmidrule(r){5-6} \cmidrule(r){7-8} \cmidrule(r){9-10}
    Method & Time(s) & $L_2$ Rel & PDE Res & $L_2$ Rel & PDE Res &  $L_2$ Rel & PDE Res & $L_2$ Rel & PDE Res\\
    \midrule
ENS(Ours)& \underline{0.43} &\textbf{8.08e-03}  & \underline{2.50e-01}& \textbf{3.42e-02}&\underline{8.05e-01} & \textbf{2.06e-02} &  \underline{3.72e-01} & \textbf{1.98e-02} & 5.55e+00 \\
FNO & \textbf{0.0093} & \underline{5.41e-02} & 1.22e+00 & 3.54e-01& 2.08e+00 & 6.66e-01 & 1.15e+01 &1.60e-01 & \underline{1.49e+00}\\
PINO & \textbf{0.0093} & 8.05e-02& 3.49e+00& 3.13e-01& 3.51e+00 & 7.04e-01 & 1.22e+01 & \underline{1.32e-01}& 4.85e+00 \\
PINO(TTOP) & 72.7 & 2.11e-01& \textbf{4.83e-02}& \underline{4.66e-02} &\textbf{3.18e-03} & \underline{2.06e-01} & \textbf{2.33e-02} & 4.45e+00 & \textbf{7.74e-02}\\
POSEIDON &6.97 & 5.62e-01 & 7.35e+00 &6.16e-01&1.19e+01 & 1.04e+00& 1.97e+01 & 6.65e-01& 1.53e+01\\
DiffusionPDE & 229.27 & 8.19e-01 & N/A & 7.37e-01 & N/A & 8.19e-01 & N/A & 9.45e-01 & N/A \\
PCFM & N/A &OOM & OOM &OOM & OOM& OOM&OOM & OOM& OOM\\
    \bottomrule
\end{tabular}
}
\end{table*}

\begin{figure*}[h!]
    \centering
    \includegraphics[width=0.84\linewidth]{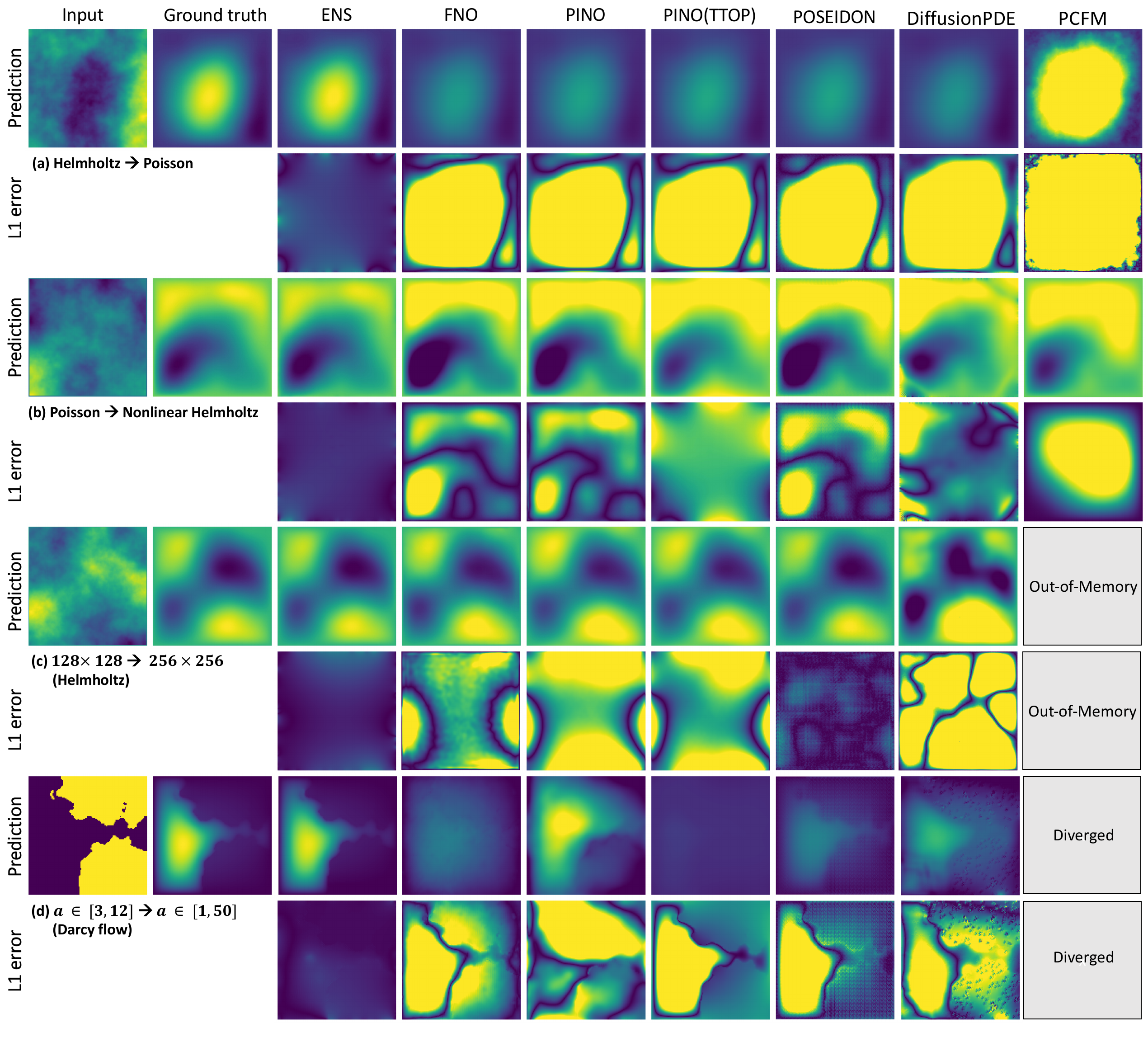}
    \caption{\textbf{Qualitative comparison on static PDEs}. Reconstructed fields (top) and point-wise $L_1$ errors (bottom) across three OOD regimes: cross-equation (HZ \ $\!\to\!$ \ PS, PS \ $\!\to\!$ \ nonlinear HZ); resolution shift (HZ, $128{\times}128\!\to\!256{\times}256$); and permeability shift (Darcy, $a \in [3,12] \!\to\! a\in [1,50]$).} 
    \label{fig: StaticPDEs}
\end{figure*}

\begin{figure*}[htb]
    \centering
    \includegraphics[width=0.84\linewidth]{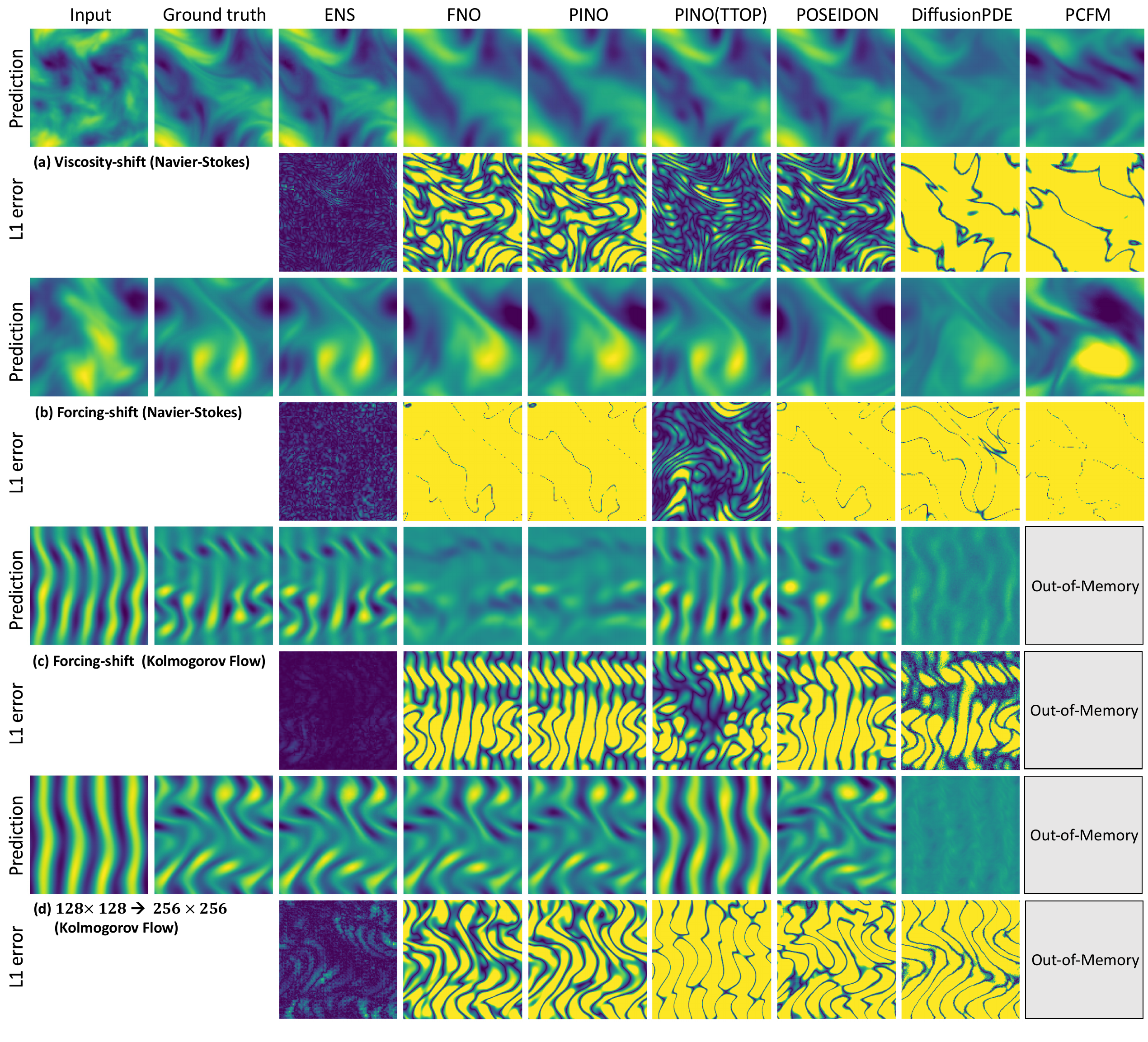}
\caption{\textbf{Qualitative comparison on dynamic PDEs.} Reconstructed fields (top) and point-wise $L_1$ errors (bottom) across three OOD regimes: viscosity shift (NS, $\nu{=}10^{-4}\!\to\!10^{-5}$); forcing shift (NS, $2\pi\!\to\!4\pi$; KF, $8\pi\!\to\!10\pi$); and resolution shift (KF, $128{\times}128\!\to\!256{\times}256$).}
\label{fig: DynamicPDEs}
        \vspace{-1em}
\end{figure*}

\begin{figure}[t]
    \centering

    \begin{minipage}[t]{0.49\linewidth}
        \centering
        \includegraphics[width=\linewidth]{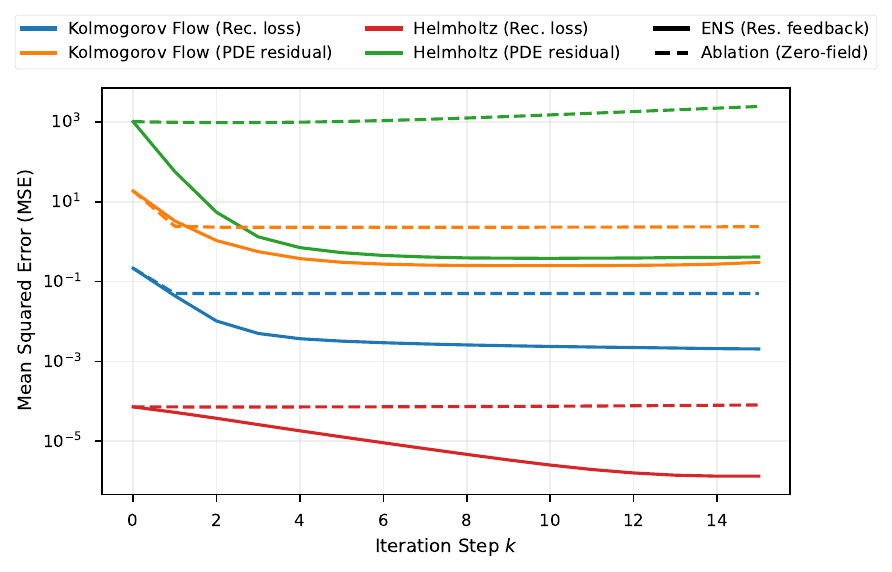}
        \captionof{figure}{\textbf{Effectiveness of residual conditioning.} We compare the loss trajectories for ENS (solid lines) against the ablation baseline (dashed lines) where zero-field is conditioned instead of the PDE residual field. The results show that ENS with residual conditioning is able to reduce both losses, whereas the ablation baseline fails to do so.}
        \label{fig:ablation}
    \end{minipage}
    \hfill
    \begin{minipage}[t]{0.49\linewidth}
        \centering
        \includegraphics[width=\linewidth]{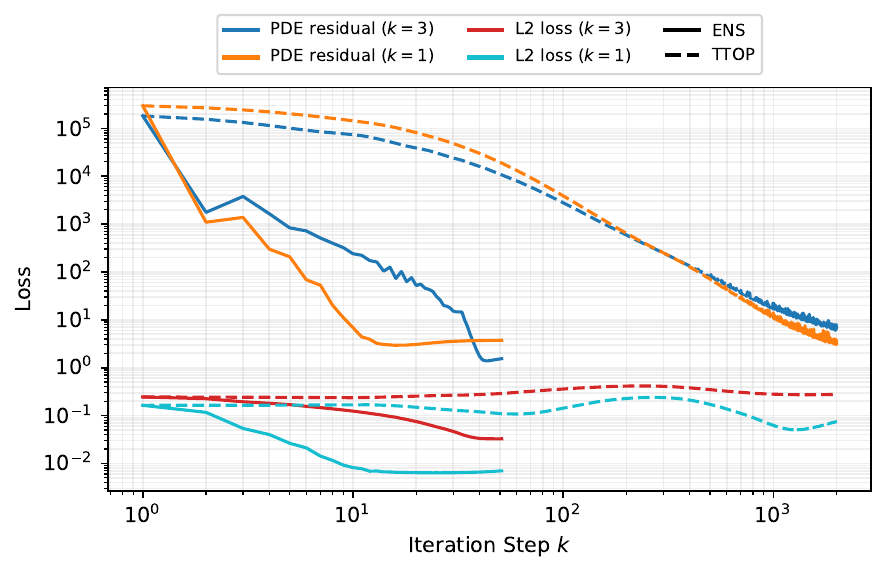}
\captionof{figure}{\textbf{PDE residual and relative $L_2$ error during test-time refinement on Helmholtz ($\kappa=1,3$).} ENS (solid) reduces both within tens of steps; PINO-TTOP (dashed) needs far more iterations and, for $\kappa=3$, its $L_2$ error rises even as its residual falls---the residual--reconstruction gap: when $\sigma_{\min}$ is small, lowering the residual may not improve the solution.}        \label{fig:jacobian}
    \end{minipage}

    \vspace{-1em}
\end{figure}

\section{Conclusion and Limitations}
We introduced Error-Conditioned Neural Solvers (ENS), a recurrent framework that treats the PDE residual field as a direct input rather than an external optimization objective, and learns a nonlinear correction policy over the residual field. Compared to existing hybrid methods, which apply first- or second-order numerical schemes to the residual, ENS attains the highest prediction accuracy in the large majority of settings across diverse PDE families. We further show that minimizing the PDE residual is an unreliable proxy for reconstruction accuracy in ill-conditioned systems, which helps explain the gap between low residuals and suboptimal reconstructions in existing hybrid methods. ENS's advantage is pronounced on distribution-shift tasks, including changed coefficients and cross-equation transfer, suggesting a path toward a general \textit{neural solver} for unseen systems. 

Our experiments are limited to relatively simple 2D systems and assume the governing equations are known at inference. Because ENS reads the residual as an input rather than minimizing it, we hypothesize that it is more forgiving of imperfect residuals from discretization error or partially known equations, which would mislead an optimization-based corrector. Extending ENS to three-dimensional problems and to settings with real, noisy observations is an exciting direction for future work.

\clearpage
{\small
\bibliographystyle{abbrvnat}
\bibliography{main}
}
\medskip


\newpage
\appendix
\clearpage

\setcounter{figure}{0}
\renewcommand{\thefigure}{A\arabic{figure}}  
\setcounter{table}{0}
\renewcommand{\thetable}{A\arabic{table}}

\section{Theoretical Analysis: Proofs and Extended Discussion}
\label{app:proofs}

This appendix provides a proof for the proposition in Sec.~\ref{sec:existing}, plus extended discussion
of the relationship between ENS and classical numerical methods. Throughout, $n$ denotes the
total degrees of freedom of the discretized system ($n = N^2$ for a 2D grid of resolution
$N \times N$, with analogous bounds for 3D).

\subsection{Proof of Prop.~\ref{prop:residual_recon}: Residual-Reconstruction
Gap}
\label{app:proof_residual_recon}

The solution manifold $\{u : r(u) \approx 0\}$ changes slowly along the direction
of the smallest singular vector of $J_r(u_{\mathrm{gt}})$ — one can move far from
$u_{\mathrm{gt}}$ while the residual barely changes. The proof constructs an
explicit perturbation exploiting this direction to create a lower bound on the worst-case reconstruction error.

\begin{proof}[\textbf{Linear case}]
Let $\mathcal{F}(u) = Au - f$ with $Au_{\mathrm{gt}} = f$. Let $v$ be the right
singular vector of $A$ corresponding to $\sigma_{\min}$, with $\|v\|_2 = 1$, and
define $\tilde{u} = u_{\mathrm{gt}} + \alpha v$. Since $A = U\Sigma V^*$:
\begin{equation}
    \|Av\|_2 = \|\Sigma e_{\min}\|_2 = \sigma_{\min}
\end{equation}
using $V^*v = e_{\min}$ and $\|U\,\cdot\,\|_2 = \|\cdot\|_2$. The residual and
reconstruction error are then exact:
\begin{equation}
    \|r(\tilde{u})\|_2 = \alpha\sigma_{\min}, \qquad
    \|\tilde{u} - u_{\mathrm{gt}}\|_2 = \alpha.
\end{equation}
Setting $\alpha = \varepsilon/\sigma_{\min}$ gives $\|r(\tilde{u})\|_2 = \varepsilon$
and $\|\tilde{u} - u_{\mathrm{gt}}\|_2 = \varepsilon/\sigma_{\min}$, a lower bound on the worst-case reconstruction error. 
\end{proof}

\begin{proof}[\textbf{Nonlinear case}]
Let $\mathcal{F}$ be twice Fr\'{e}chet-differentiable near $u_{\mathrm{gt}}$ with
$r(u_{\mathrm{gt}}) = 0$ and $\|D^2\mathcal{F}\|_{\mathrm{op}} \leq M$ on
$B_\delta(u_{\mathrm{gt}})$. Let $v$ be the right singular vector of
$J_r(u_{\mathrm{gt}})$ corresponding to $\sigma_{\min}$, $\|v\|_2 = 1$, and
define $\tilde{u} = u_{\mathrm{gt}} + \alpha v$. By Taylor expansion:
\begin{equation}
    \|r(\tilde{u})\|_2 \leq \alpha\sigma_{\min} + \tfrac{1}{2}M\alpha^2.
\end{equation}
Set $\alpha = \varepsilon/(2\sigma_{\min})$ for any $\varepsilon \in
(0,\, 4\sigma_{\min}^2/M)$. Then $\tfrac{1}{2}M\alpha^2 =
\varepsilon^2 M/(8\sigma_{\min}^2) < \varepsilon/2$, giving:
\begin{equation}
    \|r(\tilde{u})\|_2 < \varepsilon \qquad\text{but}\qquad
    \|\tilde{u} - u_{\mathrm{gt}}\|_2 = \frac{\varepsilon}{2\sigma_{\min}
    (J_r(u_{\mathrm{gt}}))},
\end{equation}
a lower bound on the worst-case reconstruction error.
\end{proof}

\paragraph{Worst-case conditioning lower bound.}
Let $r:\mathbb{R}^N\to\mathbb{R}^N$ denote a discretized residual map with
$r(u_\star)=0$, and write $J_\star=J_r(u_\star)$. Assume $J_\star$ is full rank
with singular values
\[
    \sigma_{\max}(J_\star)=L,\qquad
    \sigma_{\min}(J_\star)=\mu,\qquad
    \chi := \mathrm{cond}(J_\star)=L/\mu .
\]
After rescaling the residual, take $L=1$, so that $\mu=\chi^{-1}$. For an
iterate $u=u_\star+e$ in the local linearized regime, $r(u)=J_\star e+o(\|e\|)$.
Suppose the finite-precision linear algebra induces perturbations at effective
relative scale
\[
    \varepsilon_N := c_N u_{\rm mach},
\]
where $c_N$ captures dimension-dependent accumulation of roundoff errors.

The lower bounds follow from the least observable singular direction. Let
$v_{\min},w_{\min}$ be right and left singular vectors satisfying
\[
    J_\star v_{\min}=\chi^{-1}w_{\min},
    \qquad
    J_\star^{-1}w_{\min}=\chi v_{\min}.
\]
For a Newton correction, the exact linearized step solves
\[
    J_\star p=-r(u),
\]
and hence gives $p=-e$. However, perturbing the right-hand side by
$\zeta=\eta w_{\min}$ yields
\[
    J_\star p_{\rm comp}=-r(u)-\zeta,
    \qquad
    p_{\rm comp}=-e-\eta\chi v_{\min}.
\]
Thus the post-correction reconstruction error satisfies
\[
    \|u+p_{\rm comp}-u_\star\|_2=\eta \chi.
\]
Taking $\eta=\varepsilon_N\|e\|_2$ gives the worst-case lower bound
\[
    \frac{\|u+p_{\rm comp}-u_\star\|_2}{\|e\|_2}
    \geq
    \varepsilon_N\chi .
\]
If the perturbation instead acts on the Jacobian, the same extremal two-dimensional
model
\[
    A=\operatorname{diag}(1,\chi^{-1})
\]
with perturbation in the weak coordinate gives
\[
    \frac{\|u+p_{\rm comp}-u_\star\|_2}{\|e\|_2}
    \geq
    \frac{\varepsilon_N\chi}{1-\varepsilon_N\chi},
    \qquad
    \varepsilon_N\chi<1 .
\]

For Gauss--Newton implemented through the normal equations,
\[
    J_\star^\top J_\star p=-J_\star^\top r(u),
\]
the coefficient matrix $H=J_\star^\top J_\star$ has condition number
\[
    \mathrm{cond}(H)=\mathrm{cond}(J_\star)^2.
\]
Applying the same extremal construction to
\[
    H=\operatorname{diag}(1,\chi^{-2})
\]
and perturbing the weak coordinate by relative size $\varepsilon_N$ gives
\[
    \frac{\|u+p_{\rm GN,comp}-u_\star\|_2}{\|e\|_2}
    \geq
    \frac{\varepsilon_N\chi^2}{1-\varepsilon_N\chi^2},
    \qquad
    \varepsilon_N\chi^2<1 .
\]
When $\varepsilon_N\chi^2\geq 1$, the perturbed normal-equation matrix may
become singular or nearly singular in the least observable direction, so no
nontrivial uniform reconstruction guarantee follows from the normal-equation
solve alone.

For a $128\times128$ image, $N=16384$; using double precision and the conservative
dense-solve scaling $c_N=N$ gives
\[
    \varepsilon_N \approx 16384\cdot 2.2\times 10^{-16}
    \approx 3.6\times 10^{-12},
\]
so the worst-case amplification scales as
$3.6\times10^{-12}\chi$ for Newton and
$3.6\times10^{-12}\chi^2$ for Gauss--Newton normal equations, with the latter
becoming unreliable at the scale
\[
    \chi \gtrsim (3.6\times10^{-12})^{-1/2}
    \approx 5.3\times 10^5 .
\]
This is easily encountered in practice. 

\begin{remark} When a PDE is conditioned such that $\sigma_{\min} \to 0$, this reconstruction error can be arbitrarily large. 
\end{remark}

\begin{remark}[Instantiation for Helmholtz and Navier-Stokes]
For $\mathcal{F}(u;f) = \Delta u + \kappa^2 u + \lambda u^3 - f$, the Jacobian
$J_r(u_{\mathrm{gt}}) = \Delta + \kappa^2 I + 3\lambda u_{\mathrm{gt}}^2 I$ is
self-adjoint on $L^2(\Omega)$, so its singular values coincide with the absolute
values of its eigenvalues. As $\kappa$ grows, near-resonances between $\Delta + \kappa^2 I$
and the spatially varying term $3\lambda u_{\mathrm{gt}}^2 I$ drive
$\sigma_{\min}(J_r) \to 0$ without requiring exact cancellation. The second
derivative bound is $M = 6\lambda\|u_{\mathrm{gt}}\|_\infty$. Together these
place large-$\kappa$ Helmholtz in the ill-conditioned regime where the gap
$1/(2\sigma_{\min})$ is large, consistent with the pronounced
residual-reconstruction discrepancy in Tab.~\ref{tab:nonlinear_hz}.

For Navier-Stokes, $J_r(u_{\mathrm{gt}}) = \nu J_\Delta -
J_{\mathrm{adv}}(u_{\mathrm{gt}})$. As $\nu \to 0$ the viscous term vanishes
and $\sigma_{\min}(J_r)$ is controlled by the advection operator alone, which
is known to be ill-conditioned for smooth complex
flows~\citep{trefethen1993hydrodynamic}. 

In both cases, methods minimizing $\|r(u)\|_2$ at test time navigate along these
flat directions, achieving low residual without commensurate reconstruction
accuracy. ENS, trained on reconstruction loss, optimizes the correct objective
directly.
\end{remark}

\subsection{Discussion of Remark~\ref{rem:initialization}: Why ENS is  Initialization-Robust}
\label{app:proof_robustness}

Remark~\ref{rem:initialization} is an empirical observation rather than a
formal theorem.

\medskip\noindent\textbf{Residual convergence is initialization-independent.}
To our surprise, although our ENS network $\mathcal{C}_\phi$ is trained only on the FNO-based (from $P_\theta$) initialization, our experiment (Fig.~\ref{fig:motivation_convergence}) shows convergence to the almost identical error floor from diverse initializations. We hypothesize that, instead of learning the pure statistical mapping from input fields to output fields, our error-conditioning forces $\mathcal{C}_\phi$ to focus on reducing the residuals until the discretization limit, provided in diverse patterns during training, allowing robustness over various residual patterns and magnitude. This capability of ENS leads to its exceptional extrapolation ability and contrasts with GD and GN, whose correction directions depend
on $J_r(\hat{u}^{(k)})$, which could get arbitrarily unreliable far from $u_{\mathrm{gt}}$. We observe empirically that this initialization robustness is essential for ENS's strong extrapolation performance, which enables ENS to recover accurate solutions even when the predictor's initial estimate is far from the true solution under distribution shift. 

\medskip\noindent\textbf{Reconstruction loss is not initialization-independent.}
All initializations converge to the same residual floor but not necessarily to
the same point on the low-residual manifold $\{u : r(u) \approx 0\}$. By
Prop.~\ref{prop:residual_recon}, points on this manifold can lie at distance
up to $\varepsilon/\sigma_{\min}$ from $u_{\mathrm{gt}}$, so better initializations
yield lower reconstruction error despite reaching the same residual floor. Therefore, the convergence property is limited to residual and does not apply to reconstruction loss as the iteration increases (Fig.~\ref{fig: Losscov}). 

\subsection{Per-Step Complexity}
\label{app:computation}
\medskip\noindent\textbf{Gradient Descent.}
Computing $\nabla_u\|r(u)\|^2 = 2J_r(u)^\top r(u)$ via reverse-mode automatic
differentiation requires one forward pass ($O(n)$ for sparse $\mathcal{F}$) and one backward
pass ($O(n)$). Total: $O(n)$, and $O(n\log n)$ for a transformer architecture.

\medskip\noindent\textbf{Gauss-Newton.}
The update $(J_r^\top J_r)\delta u = -J_r^\top r$ requires:
\begin{enumerate}
    \item Forming $J_r(u) \in \mathbb{R}^{n \times n}$: $O(n^2)$ via $n$ backward passes
    \item Forming $J_r^\top J_r \in \mathbb{R}^{n \times n}$: $O(n^3)$ naively, $O(n^2)$
    exploiting sparsity
    \item Solving the $n \times n$ linear system: $O(n^3)$ directly, $O(n^2)$ via Krylov
    methods with sparse $J_r$
\end{enumerate}
Total: $O(n^2)$ with Krylov solvers, $O(n^3)$ naively. Classical solvers (multigrid, CG)
exploit the hierarchical sparsity of $\mathcal{F}$ to achieve $O(n)$ to $O(n\log n)$ for the
same system. Gauss-Newton is therefore $O(n)$ to $O(n^2)$ times more expensive per step than
classical solvers, while providing weaker convergence guarantees.

\medskip\noindent\textbf{PCFM's Schur complement.}
PCFM reduces cost by solving only the $m \times m$ Gram matrix $JJ^\top$, costing
$O(m^3 + mn)$. For $m \ll n$ (e.g., $m=1$ for scalar mass conservation), this is $O(n)$.
For full PDE residual enforcement where $m = n$, this recovers $O(n^3)$. PCFM's claimed
$O(n)$ efficiency therefore does not extend to general PDE residuals.

\medskip\noindent\textbf{ENS.}
One residual evaluation $\mathcal{F}(\hat{u}^{(k)};f)$ costs $O(n)$ for sparse $\mathcal{F}$.
One forward pass of $\mathcal{C}_\phi$ costs $O(n)$ for a U-Net or FNO backbone and
$O(n\log n)$ for a transformer-based neural operator. Total: $O(n)$ to $O(n\log n)$ per refinement step — comparable to a single 
linear Krylov solve, and $O(n)$ to $O(n^2)$ times cheaper per step than 
Gauss-Newton for the same residual dimension, and comparable to black-box feedforward neural operators like FNO.

\section{Experimental setup}
\label{app:instantiations}

\subsection{Dataset Preparation}
\label{app:PDEs}
 
All static PDE datasets (Helmholtz, Poisson, Darcy) are generated by sampling input
fields from Gaussian Random Fields (GRFs) and solving the resulting sparse linear
systems using standard sparse direct solvers; we use Newton's method for non-linear Helmholtz. Navier-Stokes initial vorticity fields
are sampled from GRFs; spatial derivatives are computed in Fourier space with
dealiasing, and time integration uses a Crank-Nicolson scheme for the viscous term
with explicit treatment of the nonlinear advection. For ENS, each static PDE and Navier-Stokes dataset contains 1,000 training samples, while Kolmogorov flow dataset contains 5,000 training samples. All datasets contain 100 test samples. Data generation times are reported in
Tab.~\ref{tab:data_generation_time}. The detailed data generation protocols for both training and in-distribution test data are described below, while the configurations for extrapolation test data are summarized in Tab.~\ref{tab:extrap}.
 

\begin{table}[t]
\centering
\caption{\textbf{Data generation time per sample across PDE families.}
Static PDE datasets are generated on CPU, while time-dependent datasets
are generated on GPU.}
\label{tab:data_generation_time}
\small
\setlength{\tabcolsep}{4pt}
\begin{tabular}{lcccccc}
\toprule
& Lin. Helm. & Poisson & Nonlin. Helm. & Darcy & N-S & Kolmogorov \\
\midrule
Time (s) & 0.04 & 0.04 & 1.00 & 0.40 & 23.30 & 15.5 \\
\bottomrule
\end{tabular}
\end{table}
 
\subsubsection{Linear/Nonlinear Helmholtz}
 
\begin{equation}
\begin{aligned}
    \Delta u(\mathbf{r}) + \kappa^2 u(\mathbf{r}) + \lambda\, u(\mathbf{r})^3
    &= f(\mathbf{r}), \quad \mathbf{r} \in \Omega \\
    u(\mathbf{r}) &= 0, \quad \mathbf{r} \in \partial\Omega
\end{aligned}
\end{equation}
The wavenumber $\kappa$ controls oscillation frequency and $\lambda$ introduces cubic
nonlinearity. For the linear case ($\lambda=0$) we train with $\kappa=1$; for the
nonlinear case we train with $\kappa=2$, $\lambda=1$. Input forcing $f$ is sampled from
a GRF with smoothness parameter $\alpha=2.0$ and length scale $\tau=3.0$ . 
 
\subsubsection{Poisson}
 
\begin{equation}
\begin{aligned}
    \Delta u(\mathbf{r}) &= a\,f(\mathbf{r}), \quad \mathbf{r} \in \Omega \\
    u(\mathbf{r}) &= 0, \quad \mathbf{r} \in \partial\Omega
\end{aligned}
\end{equation}
The scaling factor $a=2$ is used for training. Poisson and Helmholtz also serve as the mutual
\emph{cross-equation transfer target} for each other — models trained on Poisson are evaluated
on Helmholtz without fine-tuning, and vice versa, exploiting the shared Laplacian structure that a
residual-conditioned corrector can generalize across.
 
\subsubsection{Darcy Flow}
 
\begin{equation}
\begin{aligned}
    -\nabla \cdot \bigl(a(\mathbf{r})\,\nabla u(\mathbf{r})\bigr)
    &= f(\mathbf{r}), \quad \mathbf{r} \in \Omega \\
    u(\mathbf{r}) &= 0, \quad \mathbf{r} \in \partial\Omega
\end{aligned}
\end{equation}
Following the setup of FNO, the coefficient field $a(\mathbf{r})$ is generated by first sampling a GRF with a smoothness parameter $\alpha = 2.0$ and length scale $\tau = 3.0$.
\[
g \sim \mathcal{N}\!\left(0,(-\Delta+\tau^2 I)^{-\alpha}\right),
\]
and then applying a pointwise thresholding operator
\[
a(\mathbf{r})=
\begin{cases}
12, & g(\mathbf{r})\ge 0,\\
3, & g(\mathbf{r})<0.
\end{cases}
\]
resulting in a binary-valued permeability field. The forcing is kept fixed $f(\mathbf{r}) = 100$. 
 
\subsubsection{Navier-Stokes/Kolmogorov flow}
 
\begin{equation}
\begin{aligned}
    \partial_t w(\mathbf{r}, t) + v(\mathbf{r}, t)\cdot\nabla
    w(\mathbf{r}, t)
    &= \nu\,\Delta w(\mathbf{r}, t) + f(\mathbf{r}),
    \quad \mathbf{r} \in \Omega,\ t \in (t_0, T] \\
    \nabla \cdot v(\mathbf{r}, t) &= 0
\end{aligned}
\end{equation}
where $w = \nabla\times v$ is the vorticity and $\nu=10^{-4}$ is the viscosity.
During training we use sinusoidal forcing:
\begin{equation}
    f(\mathbf{r}) = 0.1\bigl(\sin(2\pi(r_1+r_2)) + \cos(2\pi(r_1+r_2))\bigr)
\end{equation}
We consider the time interval $t\in[2,6]$, where the initial vorticity $w_0$ at
$t=2$ is used to predict the next 20 frames $\{w_1,\ldots,w_{20}\}$ with a temporal resolution of
$\Delta t=0.2$. We also study the Kolmogorov flow setting, where the forcing term
is given by $f(\mathbf{r}) = -4\cos(8\pi r_2)$, the viscosity is set to $\nu = 5 \times 10^{-4}$, and solutions are collected over the interval  $t\in[2,4]$ with temporal resolution $\Delta t = 0.05$. Note that these time steps correspond to the dataset sampling intervals. The numerical simulations are performed using a much finer time step of $\Delta t=10^{-4}$.
 
\subsection{PDE Residual Computation}
\label{sec: residual_comp}
Unless otherwise specified, all PDE residuals reported in the tables are computed as the mean squared error (MSE) of the residual field.
\paragraph{Helmholtz, Poisson, Darcy.} Second-order central finite differences on a
uniform $n\times n$ grid with spacing $h=1/(n-1)$, differentiable with respect to $u$.
For nonlinear Helmholtz at grid point $(i,j)$:
\begin{equation}
    \begin{aligned}
        r_{ij} &= \frac{u_{i+1,j}-2u_{i,j}+u_{i-1,j}}{h^2}
                + \frac{u_{i,j+1}-2u_{i,j}+u_{i,j-1}}{h^2} \\
               &\quad + \kappa^2 u_{ij} + \lambda u_{ij}^3 - f_{ij}
    \end{aligned}
\end{equation}

\paragraph{Navier-Stokes and Kolmogorov Flow.}Given a predicted vorticity trajectory $w \in \mathbb{R}^{T \times N \times N}$, we compute the PDE residual using spectral differentiation in space and finite differences in time. Specifically, the temporal derivative $\partial_t w$ is approximated by forward and backward differences at the first and last frames, respectively, and by centered differences for all intermediate frames. Spatial derivatives are computed in Fourier space.  We recover the stream function by
\begin{equation}
   \begin{aligned}
   \hat{\psi}(\mathbf{k}) =
   \frac{\hat{w}(\mathbf{k})}{4\pi^2|\mathbf{k}|^2}
   \end{aligned}
\end{equation}
The velocity is then obtained as
\begin{equation}
   \begin{aligned}
   u = \partial_y \psi, \qquad v = -\partial_x \psi
   \end{aligned}
\end{equation}
and the vorticity derivatives $\partial_x w$, $\partial_y w$, and $\Delta w$ are also computed spectrally. The residual is defined as
\begin{equation}
   \begin{aligned}
   r = \partial_t w + u \partial_x w+v \partial_y w - \nu \Delta w -f
   \end{aligned}
\end{equation}

\subsection{Extrapolation Protocols}
 
All evaluations are zero-shot — no fine-tuning on target distributions. We summarize the extrapolation settings in Tab.~\ref{tab:extrap}.

\begin{table*}[!htbp]
\centering
\caption{ Training vs.\ evaluation parameters for each extrapolation regime.}
\label{tab:extrap}
\begin{tabular}{llll}
\toprule
\textbf{PDE} & \textbf{Regime} & \textbf{Training} & \textbf{Evaluation} \\
\midrule
linear Helmholtz     & Extrapolation      & $\kappa =1$    & $\kappa=3$ \\
linear Helmholtz     & Super-resolution   & $128\times128$      & $256\times256$ \\
linear Helmholtz     & Cross-equation   & Helmholtz   &  Poisson\\
\midrule
nonlinear Helmholtz     & Extrapolation      & $\kappa=2, \lambda=1$    & $\kappa=3, \lambda=4$ \\
nonlinear Helmholtz     & Super-resolution   & $128\times128$      & $256\times256$ \\
\midrule
Poisson       & Scaling shift      & $a=2$               & $a=1$ \\
Poisson       & Super-resolution   & $128\times128$      & $256\times256$ \\
Poisson       & Cross-equation   & Poisson    & nonlinear Helmholtz\\
\midrule
Navier-Stokes & Viscosity shift    & $\nu=10^{-4}$       & $\nu=10^{-5}$ \\
Navier-Stokes & Forcing shift      & $2\pi$ freq.        & $4\pi$ freq. \\
Navier-Stokes & Super-resolution   & $128\times128$      & $256\times256$ \\
\midrule
Darcy  & Permeability shift & $a\in[3,12]$ & $a\in[1,50]$ \\
Darcy flow       & Super-resolution   & $128\times128$      & $256\times256$ \\
 \midrule
 Kolmogorov flow & Forcing shift &$8\pi$ freq.& $10\pi$ freq. \\
Kolmogorov flow & Viscosity shift & $\nu = 5 \times 10^{-4}$ & $\nu = 10^{-3}$ \\
Kolmogorov flow& Super-resolution   & $128\times128$      & $256\times256$ \\
\bottomrule
\end{tabular}
\end{table*}

\subsection{ENS Architecture.} 
The predictor $\mathcal{P}_\theta$ maps the PDE inputs $(f, g)$ to an initial estimate $\hat{u}^{(0)}$ in a single forward pass. At each step $k$, the corrector $\mathcal{C}_\phi$ takes $(f, g, \hat{u}^{(k)}, r^{(k)})$, concatenated along the channel dimension, and outputs a correction field $\delta\hat{u}^{(k)}$ at the same spatial resolution. The predictor and corrector use the same backbone, which depends on the problem family. For the static equations (Helmholtz, Darcy, Poisson) we use a Fourier Neural Operator (FNO) with four Fourier layers, 64 hidden channels, and 20 Fourier modes per spatial dimension; the standard FNO lifting and projection modules are replaced with 3-layer CNNs to better preserve local spatial structure (5-layer CNNs for Darcy). For the turbulent flows (Navier--Stokes and Kolmogorov), both networks instead use a transformer-based VideoPDE backbone~\citep{li2025videopde}, which better captures the spatio-temporal structure of these systems.

\subsection{Training and inference}
Both networks are trained jointly end-to-end by unrolling the full $K$-step
correction loop. At each training step: (i) $\mathcal{P}_\theta$ produces
$\hat{u}^{(0)}$, (ii) residuals $r^{(k)}$ are
recomputed at each step as described in Sec.~\ref{sec: residual_comp}, (iii) all $K$ intermediate
predictions are supervised:
\begin{equation}
    \mathcal{L}_{\mathrm{data}} = \frac{1}{K}\sum_{k=1}^{K}
    \frac{1}{Bn}\sum_{i=1}^{B}\left\|\hat{u}^{(k)}_i - u_{\mathrm{gt},i}\right\|_2^2,
\end{equation}
where $B$ is the batch size.
Supervising all intermediate steps rather than only the final output provides
dense gradient signal to both networks and exposes $\mathcal{C}_\phi$ to the
full distribution of residual fields it encounters at inference time, including
compounding errors from imperfect early corrections. We use $K = 5$ unrolled
steps during training. Both networks are optimised jointly with AdamW
($\mathrm{lr} = 1\times10^{-4}$, weight decay $10^{-4}$). All experiments are conducted on a single RTX 6000 Ada with a batch size $B = 20$. We set the correction step size to $\beta=0.05$ for all static PDEs except the Poisson equation, for which we use $\beta=0.03$. For time-dependent PDEs, we use $\beta=0.05$ for Kolmogorov flow and $\beta=0.1$ for Navier–Stokes.

At inference time, the loop runs for $K_{\mathrm{inf}}$ steps determined by
the computational budget and the amount of distribution shift, with $r^{(k)}$ recomputed from the updated solution
at every step. For out-of-distribution experiments we increase $k$ to address the slower convergence. Specifically, for the Helmholtz equation, we increase $k$ from $15$ in-distribution to $80$ for wave-number extrapolation and to $120$ for super-resolution, with the inference stopping point determined by PDE residual convergence (Fig.~\ref{fig: Losscov}). For Navier--Stokes, we increase $k$ from $5$ to $15$ for viscosity shift and to $10$ for forcing shift. For Kolmogorov flow, we increase $k$ from $12$ to $15$ for both viscosity and forcing shifts, and to $90$ for super-resolution. Performance continues to improve beyond $k = 5$ until the
residual floor $\varepsilon_{\mathcal{C}}$ is reached. The corrector seems robust
to imperfect initialization quality since it was trained with simulated
residuals across the full range of correction stages.

\begin{figure*}[!htbp]
\centering
\includegraphics[width=0.8\linewidth]{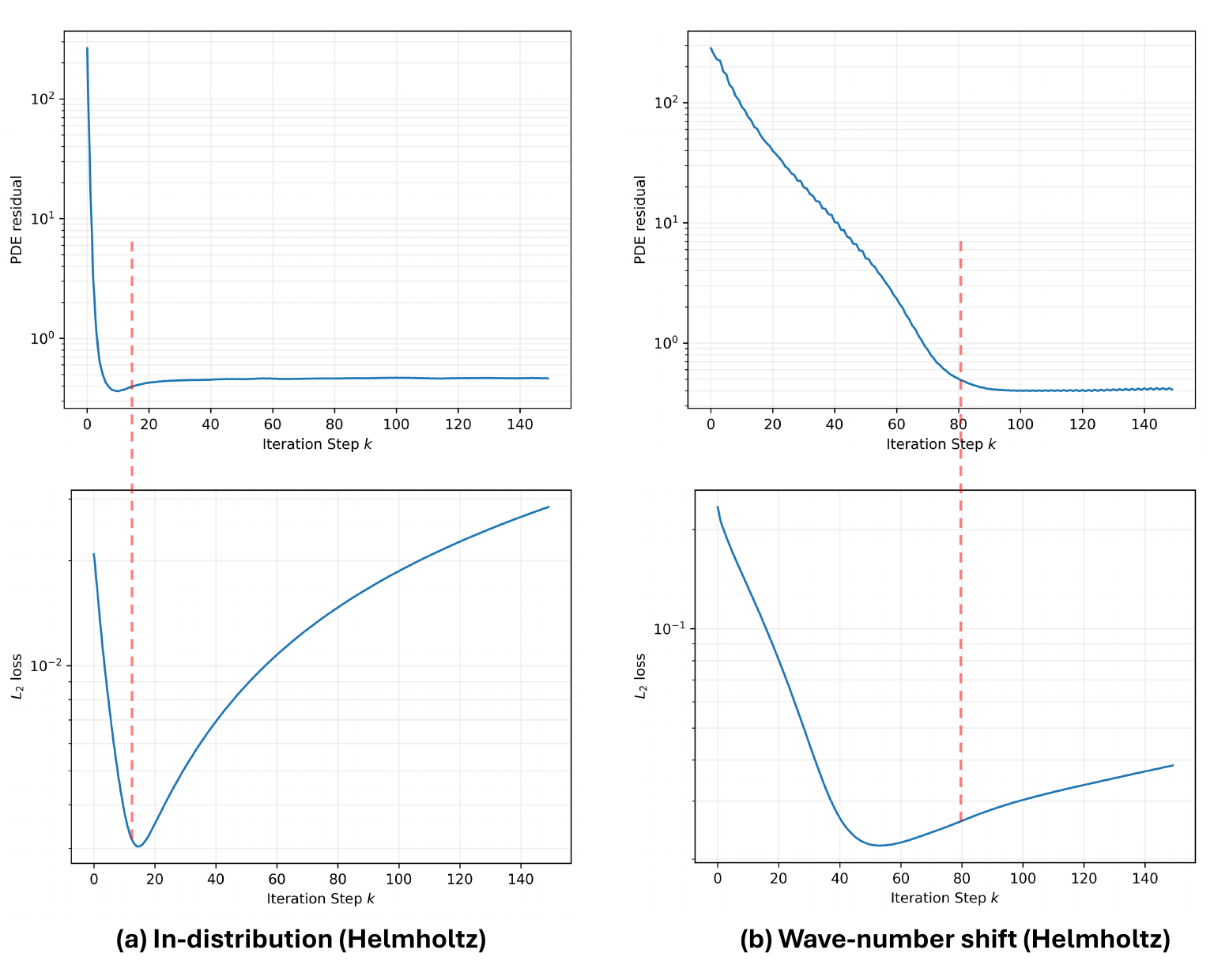}
\caption{\textbf{PDE
residual and $L_2$ loss over ENS iterations}. The PDE residual decreases and then converges, whereas the relative \(L_2\) loss, after reaching its minimum, starts to increase. This suggests a residual–reconstruction mismatch in later iterations. We therefore terminate inference when the PDE residual first starts to converge (red dashed line). The number of iteration step is therefore decided empirically on a small held-out validation samples.}
\label{fig: Losscov}
\end{figure*}

\subsection{Baseline Details}
For all baselines, we follow their official implementations whenever possible. For static PDEs, FNO and PINO employ a 4-layer FNO-2D backbone with 64 hidden channels and 20 Fourier modes, while for dynamic PDEs they use a 4-layer FNO-3D backbone with 64 hidden channels and 8 Fourier modes. Following the official configurations, FNO is trained using only the reconstruction loss (relative $H^1$ loss). PINO is trained with reconstruction loss and PDE residual for static PDEs, and additionally incorporates a first-frame regularization term for dynamic PDEs. PINO-TTOP initializes from the pretrained PINO model and performs test-time optimization using only the PDE residual for static PDEs, and the PDE residual together with first-frame regularization for dynamic PDEs, all following their official implementation. We use 500 optimization steps for static PDEs and 1,000 steps for dynamic PDEs as in codebase. For each extrapolation setting, PINO-TTOP is re-run from the pretrained PINO model and optimized independently. For POSEIDON, we initialize from the pretrained POSEIDON-B model and fine-tune it to be converged for 200 epochs on each downstream task.

Our DiffusionPDE baselines for time-dependent PDEs (NS, KF) use a two-frame representation (modeling only initial and final states as a joint distribution) in contrast to the full spatiotemporal fields used by PCFM and ENS. We reuse author-provided checkpoints for Darcy, Poisson, and linear Helmholtz; we train custom DiffusionPDE models for other PDEs, fine-tuning the NS model from the author's non-bounded NS weights. Following the original implementation, we use a Heun EDM sampler with 2,000 steps and DPS-style soft guidance on the PDE residual and input observation at each denoising step.

Our PCFM baseline uses the original Functional Flow Matching (FFM) framework with an FNO backbone (including 3D FNO for time-dependent PDEs). We train all PCFM models from scratch following the original training hyperparameters. During inference (50 ODE steps), PDE constraints are enforced via a hard projection at every integration step using a Gauss-Newton solver; we use the authors' default sampler without guided interpolation. For time-dependent PDEs (NS, KF), the PCFM applies IC and mass conservation constraints, following the original authors' implementation.

For the Navier--Stokes equations, we further compare our method with PhysicsCorrect. We train and evaluate PhysicsCorrect on the same ENS dataset. To ensure a fair comparison and follow their original experimental setting, we use the first frame as the initial condition to autoregressively roll out the remaining frames. Since PhysicsCorrect operates on the stream function, we convert the vorticity field $\omega$ to the stream function $\psi$ by solving the Poisson equation
$\nabla^2 \psi = -\omega$. Following the original PhysicsCorrect implementation, we additionally apply the physics-based correction after each rollout step before proceeding to the next prediction.

\section{Full results}

\label{app: addition}
We report remaining qualitative and quantitative results that we omitted in the main paper. 

\begin{table*}[!htbp]
  \caption{\textbf{Forward prediction results for the Poisson problem}. Bold indicates the best result, and underlining indicates the second-best. }
  \label{tab:forward_poisson}
  \centering
  \small 
  \begin{tabular}{lcccccccc}
    \toprule
    & \multicolumn{2}{c}{In-distribution} & \multicolumn{2}{c}{Super-resolution} & 
    \multicolumn{2}{c}{Extrapolation} & 
    \multicolumn{2}{c}{Cross-equation}
    \\
    \cmidrule(r){2-3} \cmidrule(r){4-5} \cmidrule(r){6-7} \cmidrule(r){8-9}
    Method & $L_2$ Rel & PDE Res & $L_2$ Rel & PDE Res & $L_2$ Rel & PDE Res& $L_2$ Rel & PDE Res\\
    \midrule
ENS(Ours) & \textbf{3.63e-03} & \textbf{3.61e+00} & \textbf{1.77e-02} & \textbf{1.93e+02} & \textbf{3.57e-02} & \textbf{4.40e+00} & 
\textbf{4.19e-02}  & \textbf{9.40e-01} \\

FNO & \underline{1.79e-02} & 6.98e+02 & \underline{2.00e-02}& 1.51e+04 & 1.00e+00 & 1.59e+03 & 5.18e-01 &  3.70e+03\\

PINO & 7.01e-02 & 5.04e+02 & 6.52e-02 & 1.24e+04 &  9.99e-01 & 1.13e+03 &  5.32e-01  & 1.39e+03  \\

PINO(TTOP) & 1.37e-01 & 8.63e+01 & 2.15e-01 & \underline{4.63e+02} & \underline{3.58e-01} & \underline{2.41e+01} &  \underline{4.64e-01} & \underline{5.27e+01}\\

POSEIDON & 3.28e-02 & 8.90e+04 & 3.75e-02 & 2.32e+05 & 9.94e-01 & 9.55e+04 & 4.80e-01 & 3.27e+04  \\
DiffusionPDE &6.24e-02 & 5.61e+03& 7.83e-01& 4.15e+04& 9.95e-01&6.92e+03 & 5.17e-01  & 1.76e+04 \\
PCFM & 8.20e-01 & \underline{2.15e+01} & OOM & OOM & 2.65e+00 & 5.42e+03 & 1.48e+00 & 2.01e+02  \\

    \bottomrule
  \end{tabular}
\end{table*}

\begin{table*}[!htbp]
\centering
\caption{\textbf{Forward prediction results for the Darcy flow problem}. Bold indicates the best result, and underlining indicates the second-best. The PCFM projection for the Darcy constraint was numerically
unstable and caused the intermediate latent to increase geometrically and overflow, indicated by "diverged". As shown in Fig. ~\ref{fig: Darcydiv}, this instability emerges during the later flow-matching steps.}
\label{tab:Darcy}
\small
\resizebox{0.8\linewidth}{!}{
\begin{tabular}{lccccccc}
    \toprule
    & \multicolumn{2}{c}{In-distribution} &
    \multicolumn{2}{c}{Extrapolation} &
    \multicolumn{2}{c}{Super-resolution}
    \\
    \cmidrule(r){2-3} \cmidrule(r){4-5} \cmidrule(r){6-7}
    Method & $L_2$ Rel & PDE Res &  $L_2$ Rel & PDE Res &  $L_2$ Rel & PDE Res \\
    \midrule
ENS(Ours) &  \textbf{1.08e-03}& \underline{3.72e+03}  &\textbf{6.84e-02} & 9.99e+06& \textbf{7.01e-03} & 3.84e+04 \\
FNO &\underline{9.45e-03}& 1.42e+05& 6.28e-01 & 6.55e+07 & \underline{1.29e-02} &1.28e+06 \\
PINO &  7.06e-02 & 2.95e+04 & 7.14e-01 &\underline{9.02e+06} & 6.79e-02 &2.37e+05 \\
PINO(TTOP) &2.35e-01 &\textbf{1.51e+03} &7.79e-01 &
\textbf{7.98e+03} &5.74e-01 & \textbf{6.24e+03}\\
POSEIDON & 1.94e-02 & 2.54e+06&\underline{5.29e-01} & 9.30e+09& 1.83e-02  &6.44e+06\\
    DiffusionPDE & 7.52e-02 &  5.15e+05& 6.69e-01 & 2.96e+11 & 5.40e-01 &
    \underline{7.48e+03}\\
    PCFM & 1.90e-01 & 2.21e+04 & diverged & diverged & OOM & OOM \\
    \bottomrule
\end{tabular}
}
\end{table*}

\begin{table*}[!htbp]
  \caption{\textbf{Forward prediction results for the nonlinear Helmholtz problem}. Bold indicates the best result, and underlining indicates the second-best.}
  \label{tab:nonlinear_hz}
  \centering
  \small 
  \begin{tabular}{lcccccc}
    \toprule
    & \multicolumn{2}{c}{In-distribution} & \multicolumn{2}{c}{Extrapolation} & 
    \multicolumn{2}{c}{Super-resolution}
    \\
    \cmidrule(r){2-3} \cmidrule(r){4-5} \cmidrule(r){6-7} 
    Method & $L_2$ Rel & PDE Res &  $L_2$ Rel & PDE Res& $L_2$ Rel & PDE Res\\
    \midrule
ENS(Ours) & \textbf{5.06e-03} & \underline{6.42e+00}&\textbf{6.31e-02}& 1.76e+01 & \textbf{1.09e-02}& \textbf{3.55e+01}\\
FNO &\underline{2.48e-02} & 3.39e+02 & 2.47e-01 & 2.59e+02 & \underline{1.33e-01} & 1.25e+04  \\
PINO &  4.51e-02 & 6.14e+01 & 2.46e-01 & 5.66e+01 & 5.94e-02 & 3.41e+03\\

PINO(TTOP) & 7.80e-02 & 1.45e+01 &   1.98e-01 & \underline{1.29e+01} &  1.83e-01 & \underline{1.22e+02} \\

POSEIDON &7.02e-02  &4.57e+04 & 2.77e-01&3.48e+04 & 4.16e-01 & 9.72e+05\\

DiffusionPDE &1.36e-01 &2.60e+05 & 4.64e-01& 5.27e+01 & 7.34e-01 & 6.04e+05 \\

PCFM & 7.89e-02& \textbf{2.53e+00} &\underline{1.75e-01} & \textbf{5.53e+00}&OOM &OOM \\

    \bottomrule
  \end{tabular}
\end{table*}

\begin{figure*}[!htbp]
\centering
\includegraphics[width=1.0\linewidth]{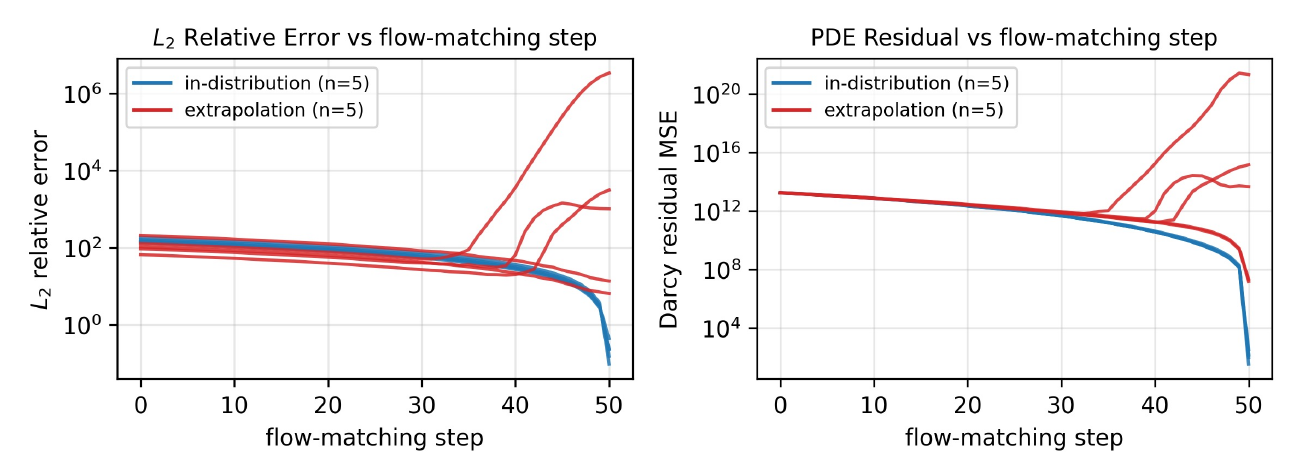}
\caption{\textbf{PCFM divergence under Darcy flow extrapolation.} The \(L_2\) relative error (left) and PDE residual MSE (right) are shown across flow-matching steps. In-distribution samples remain stable, whereas several extrapolation samples diverge during later sampling steps, with rapidly increasing reconstruction error and PDE residual.}
\label{fig: Darcydiv}
\end{figure*}

\begin{figure*}[!htbp]
    \centering
    \includegraphics[width=1.0\linewidth]{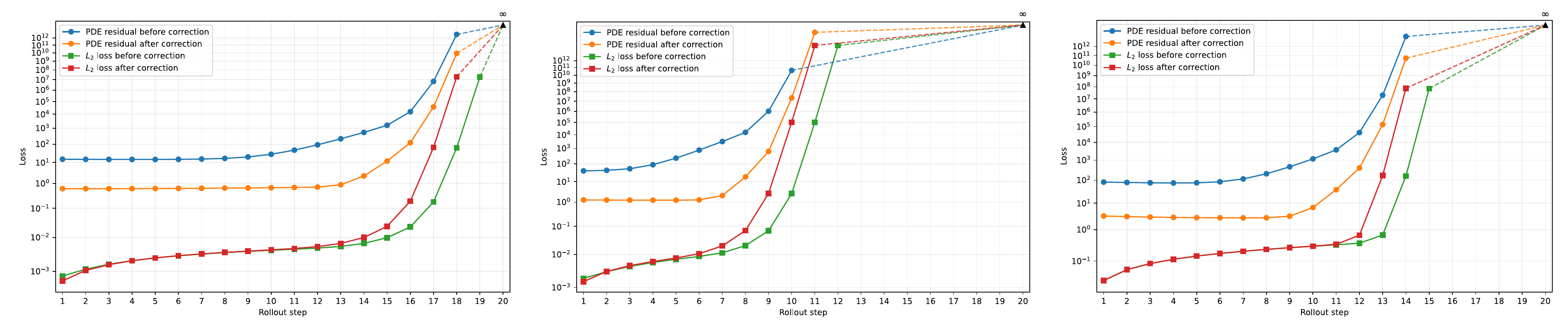}
    \caption{\textbf{$L_2$ loss and PDE residual before and after each rollout step.} The three panels, from left to right, show the in-distribution, viscosity-shift, and forcing-shift evaluation settings. For the Navier-Stokes equation with lower viscosity, a residual-reconstruction gap emerges: although the physics correction consistently reduces the PDE residual at each step, it fails to reduce the reconstruction error. Consequently, the correction-induced reconstruction error accumulates over the autoregressive rollout, eventually causing both the reconstruction loss and PDE residual to diverge.}
    \label{fig: PhyCorr}
\end{figure*}

\begin{figure*}[!htbp]
\centering
\includegraphics[width=1.0\linewidth]{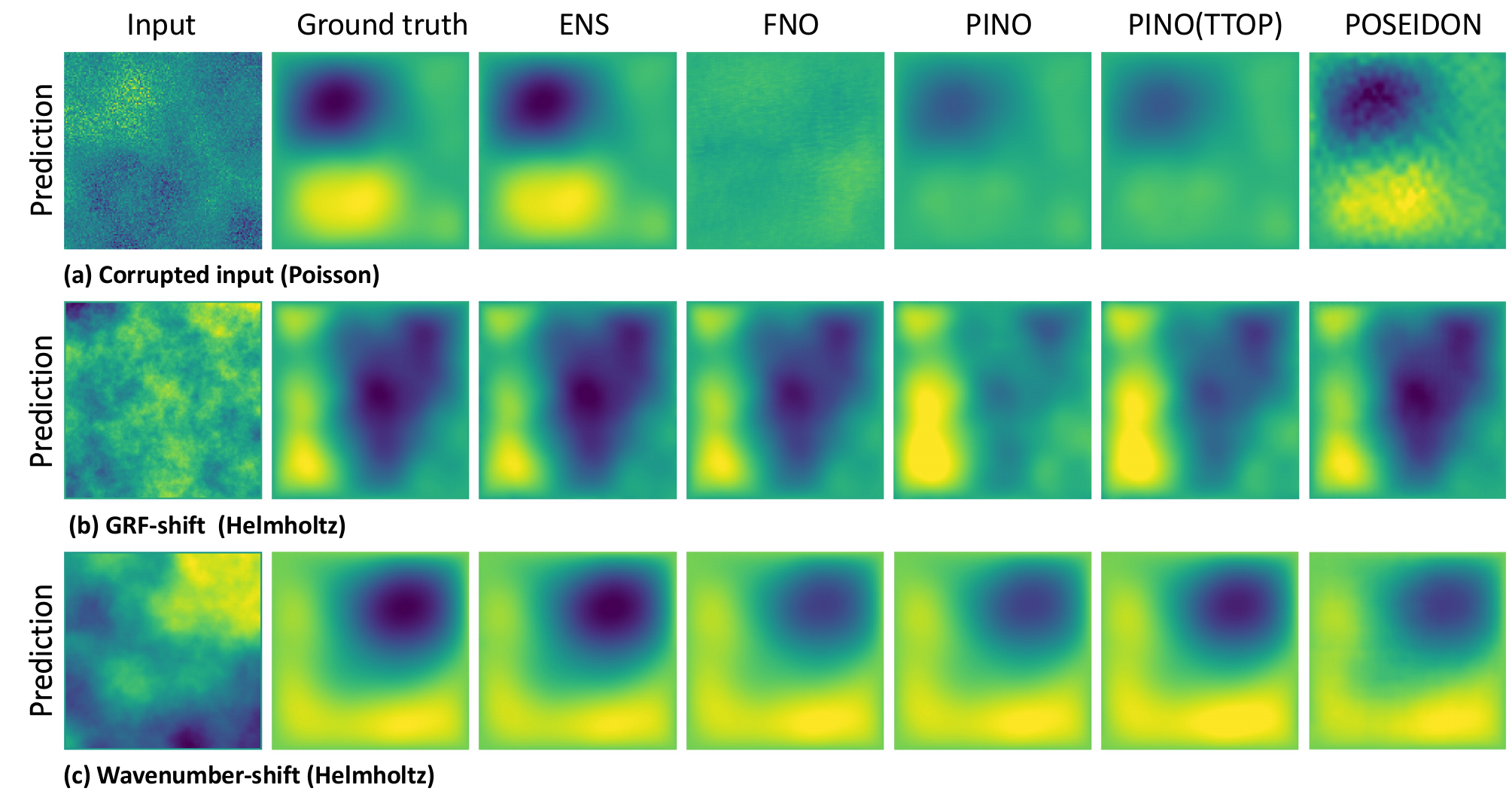}
\caption{\textbf{Additional qualitative comparison on static PDEs}. Reconstructed fields across three OOD regimes: corrupted input field (Poisson); input field (GRF) shift (Helmholtz); and wavenumber shift (Helmholtz).}
\label{fig: Staticapp}
\end{figure*}

\begin{figure*}[!htbp]
\centering
\includegraphics[width=1.0\linewidth]{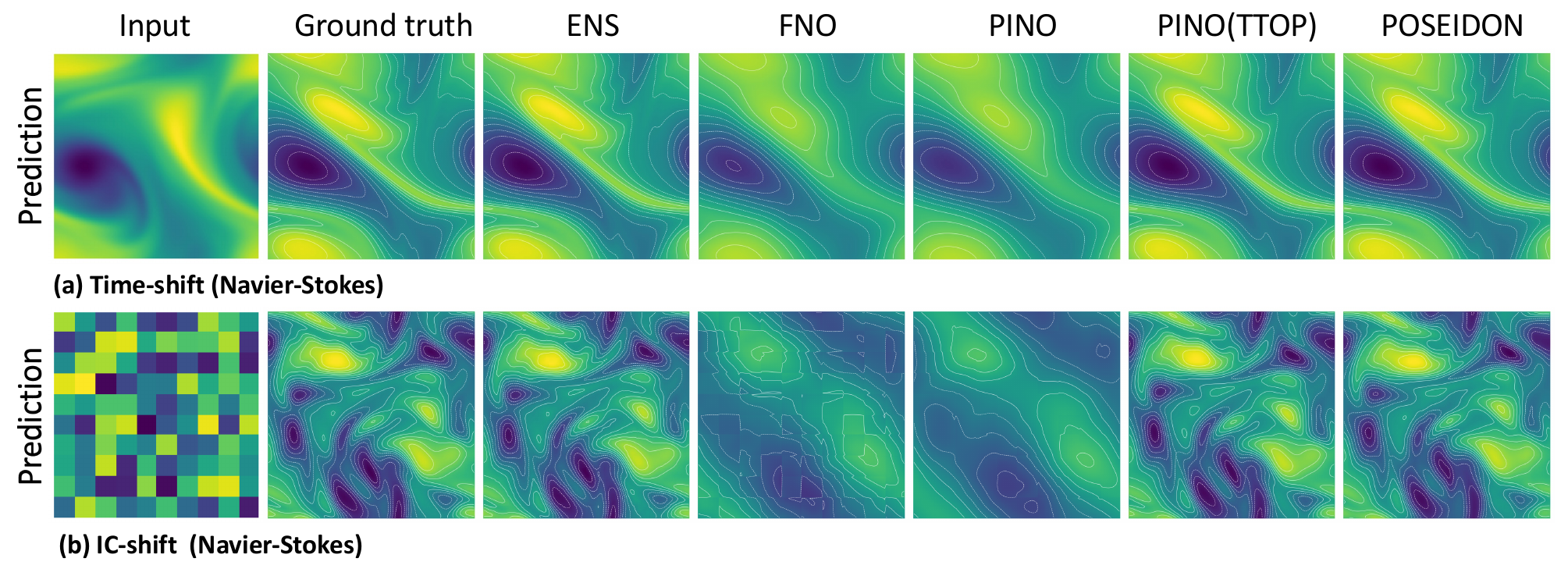}
\caption{\textbf{Additional qualitative comparison on dynamic PDEs}. Reconstructed fields across two OOD regimes: time range shift and initial condition shift. }
\label{fig: Dynacapp}
\end{figure*}

\begin{figure*}[!htbp]
    \centering
    \includegraphics[width=0.99\linewidth]{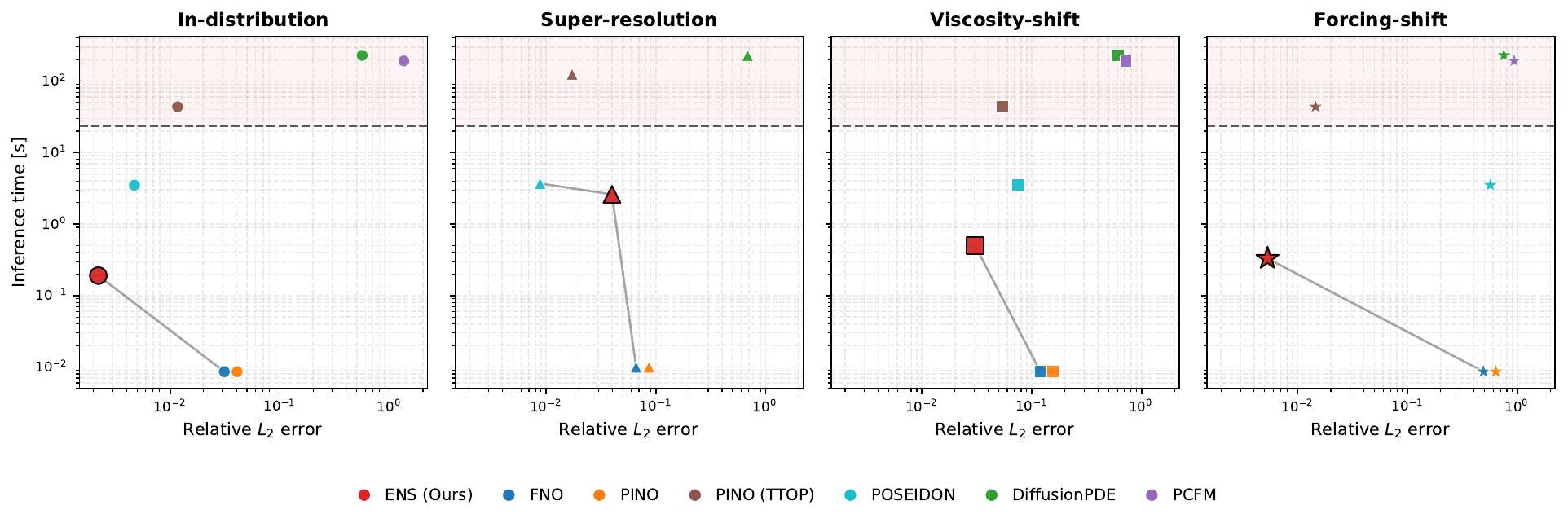}
    \caption{\textbf{Accuracy vs.\ inference cost on Navier--Stokes, per regime.} Lower-left is better; the gray curve in each panel traces that regime's Pareto frontier and the dashed line marks the classical solver ($\approx$23s/sample). ENS is the most accurate learned method on the frontier for in-distribution, viscosity-shift, and forcing-shift, while POSEIDON is more accurate on super-resolution at higher cost. ENS's inference time is 0.19s in-distribution, 0.5s under viscosity-shift, 0.33s under forcing-shift, and 2.6s for super-resolution, as the number of correction steps adapts to difficulty. The correction and diffusion hybrids (PINO-TTOP, DiffusionPDE, PCFM) lie above the solver line, i.e.\ slower than solving the PDE directly. FNO and PINO are offset horizontally by $\pm$12\% for visibility, as their accuracy is nearly identical.} 
    \label{fig:pareto}
\end{figure*}

\begin{figure*}[!htbp]
  \centering
  \includegraphics[width=\textwidth]{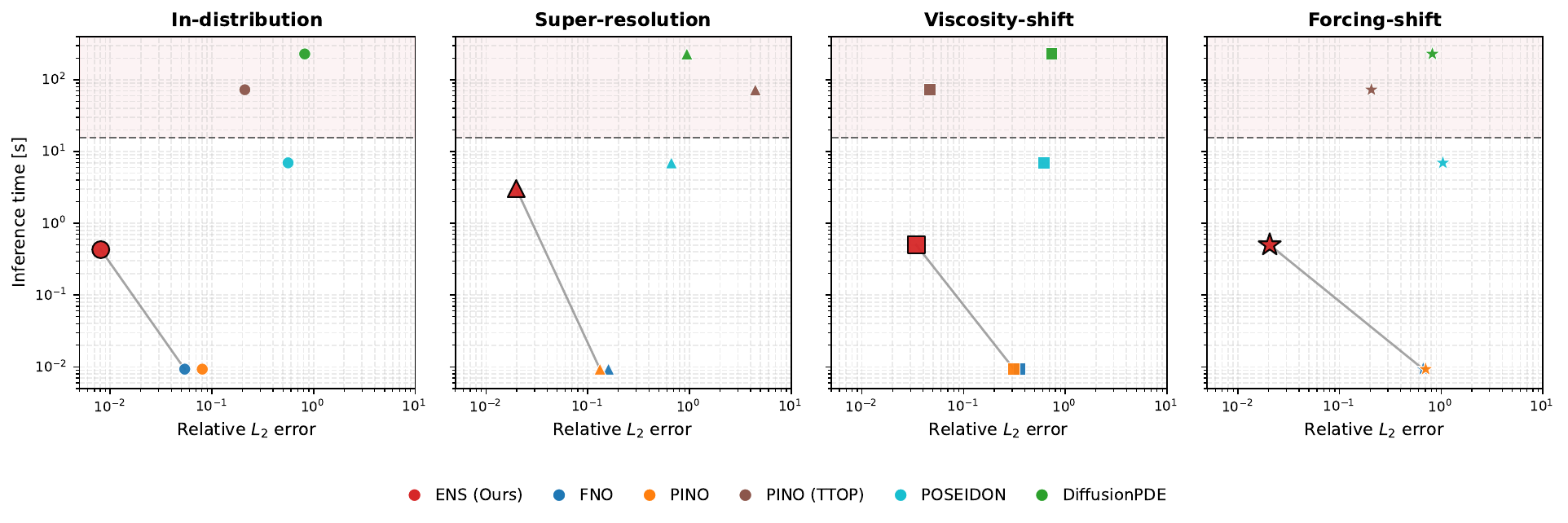}
  \caption{\textbf{Accuracy vs.\ inference cost on Kolmogorov flow, per regime.} Lower-left is better;
  the gray curve in each panel traces that regime's Pareto frontier and the dashed line marks the
  classical solver ($\approx$\textbf{15.5}\,s/sample). ENS is the most accurate learned method on the
  frontier in all four regimes, with inference time 0.43\,s in-distribution, 0.50\,s under viscosity-
  and forcing-shift, and 3.05\,s for super-resolution, as the number of correction steps adapts to
  difficulty. PINO-TTOP reaches low PDE residuals but high reconstruction error
  (Prop.~\ref{prop:residual_recon}), seen as its rightward position despite test-time optimization
  and most extreme under super-resolution ($L_2\approx4.5$). PCFM runs out of memory in the 40-frame
  setting and is omitted.}
  \label{fig:pareto_kf}
\end{figure*}

\section{Additional Ablation Analysis}
\label{app: addition}

We expand on the ablations summarized in the main text, covering (i) residual conditioning versus computational depth, (ii) matched-backbone comparisons that isolate the method from the architecture, (iii) backbone expressivity, (iv) the dynamics of the learned correction, and (v) residual versus gradient conditioning.

\paragraph{Residual conditioning versus computational depth.} To confirm that ENS's improvement comes from reading the residual rather than from additional iterations, we replace the residual input with a zero field while holding the architecture and step count fixed. Figs.~\ref{fig: RLvis} and~\ref{fig: PDEvis} show the reconstruction error and PDE residual over correction steps: the ablated model reduces neither quantity, whereas ENS drives both down, so the effect comes from residual conditioning rather than the recurrent computation itself. The same figures show that out-of-distribution inputs reach the same floors as in-distribution inputs but require more correction steps, which the residual-plateau stopping rule accommodates automatically.

\paragraph{The architecture is not the source.} ENS uses a CNN-augmented FNO backbone for the static PDEs and a transformer-based VideoPDE backbone for the turbulent flows. To rule out the backbone as the source of the gains, we place the strongest baselines---the plain operator, its physics-informed (PINO) variant, and its test-time-optimized (TTOP) variant---on ENS's own backbone for each family (Tabs.~\ref{tab:HZ_AB},~\ref{tab:NS_AB},~\ref{tab:KF_AB}). Across Helmholtz, Navier--Stokes, and Kolmogorov flow, the baselines methods do not show clear improvement against their original FNO architectures with mixed signals. ENS retains the lowest reconstruction error in every regime against the baselines that share its architecture. The TTOP variant reproduces the residual--reconstruction gap of Prop.~\ref{prop:residual_recon} on the matched backbone: it reaches comparable or lower PDE residuals while leaving reconstruction error high, and the gap widens with ill-conditioning---mild on Navier--Stokes, where VideoPDE(TTOP) attains a competitive $L_2$ of $8.69\mathrm{e}{-3}$, and severe on Kolmogorov flow, where VideoPDE(TTOP) attains the lowest residual yet an $L_2$ of $4.34\mathrm{e}{-1}$, roughly $50\times$ worse than ENS (in-distribution). The ENS's advantage is therefore attributable to residual conditioning, not the backbone.

\paragraph{Backbone expressivity.} Residual conditioning nonetheless requires a backbone able to represent the high-frequency residual field. Tab~\ref{tab:Ablation} sweeps backbones on the linear Helmholtz equation: a pure FNO diverges, its CNN-augmented variants converge (CNN-FNO-CNN best, and in the fewest steps), and a U-Net converges but to a substantially worse solution. Spectral truncation prevents a pure FNO from representing the residual it is asked to read, while the local feature extraction added by CNN lifting and projection restores this capacity. ENS thus works with any sufficiently expressive backbone rather than a single specific one, including the transformer-based architecture applied to the turbulent flows.

\paragraph{Dynamics of the learned correction.} Figs.~\ref{fig: delta_u} and~\ref{fig: delta_u_noisy} visualize the reconstruction loss, PDE residual, and the learned correction $\delta\hat{u}^{(k)}$ across iterations for extrapolated and Gaussian-noised source fields. The correction magnitude is largest in early iterations and decays as the prediction approaches the solution manifold; under distribution shift, convergence requires roughly $60$--$90$ iterations, well beyond the in-distribution case. That ENS recovers accurate solutions even from noised inputs---initializations far from any seen in training---indicates that the corrector learns a robust correction policy rather than a fixed mapping from the training distribution.

\paragraph{Residual versus gradient conditioning.} Finally, we ask whether the corrector benefits from the residual field itself or from its physics-loss gradient $\nabla_{u_k}\mathcal{L}_{\mathrm{res}}$, with $\mathcal{L}_{\mathrm{res}}=\tfrac12\|\mathcal{R}(u_k)\|_2^2$. Fig.~\ref{fig: gradient} compares the two: gradient conditioning accelerates residual minimization but stalls the reconstruction loss, whereas residual conditioning reduces both together. Conditioning on the gradient biases the corrector toward driving the residual to zero---the very objective Prop.~\ref{prop:residual_recon} shows to be an unreliable proxy---confirming that the residual is most useful as an input signal, not as an optimization target.

\begin{table}[t]
  \caption{\textbf{Comparative analysis of ENS across various architectural backbones for the linear Helmholtz equation (In-Distribution)}. Bold indicates the best result, and underlining indicates the second-best.}
  \label{tab:Ablation}
  \centering
  \small
  \setlength{\tabcolsep}{4pt}

  \begin{tabular}{lcccc}
    \toprule

    & \multicolumn{4}{c}{Helmholtz} \\

    \cmidrule(r){2-5}

    Model 
    & $L_2$ Rel  & PDE Res  & $k$ & Convergence \\
    
    \midrule

    CNN-FNO-CNN 
    & \textbf{3.04e-03} & \textbf{3.90e-01} & \underline{15} & \ding{51} \\

    FNO 
    & 3.16e-02 & 9.03e+03 & - & \ding{55} \\

    CNN-FNO 
    & \underline{4.66e-03} & 5.19e+01 & 20 & \ding{51} \\

    FNO-CNN
    & 1.44e-02 & \underline{9.19e+00} & 45 & \ding{51} \\

    U-net
    & 1.17e-01 & 7.20e+03 & 20 & \ding{51} \\

    POSEIDON
    & 1.31e-02 &  6.21e+03 & \textbf{8} & \ding{51} \\

    \bottomrule
  \end{tabular}
\end{table}

\begin{table*}[!htbp]
  \caption{\textbf{Ablation study on architecture for Helmholtz}. Bold indicates the best result, and underlining indicates the second-best.}
  \label{tab:HZ_AB}
  \centering
  \scriptsize
  \setlength{\tabcolsep}{3pt}
  \resizebox{\linewidth}{!}{%
  \begin{tabular}{lccccccccc}
    \toprule
    & \multicolumn{2}{c}{In-distribution} & \multicolumn{2}{c}{Wavenumber-shift} & 
    \multicolumn{2}{c}{Cross-equation} & 
    \multicolumn{2}{c}{Super-resolution}
    \\
    \cmidrule(r){2-3} \cmidrule(r){4-5} \cmidrule(r){6-7} \cmidrule(r){8-9}
    Method  & $L_2$ Rel & PDE Res & $L_2$ Rel & PDE Res & $L_2$ Rel & PDE Res& $L_2$ Rel & PDE Res\\
    \midrule
ENS(Ours) & \textbf{3.04e-03} & \underline{3.90e-01} &  \textbf{2.59e-02} & \underline{5.19e-01} & \textbf{8.73e-02} & \underline{6.54e+01} &
\textbf{1.82e-02} & \textbf{4.52e-01} \\

CNN-FNO-CNN & \underline{1.80e-02} & 5.29e+02  & 2.37e-01 & 5.26e+02 &\underline{5.08e-01} & 2.51e+04& \underline{1.37e-01} & 3.54e+04\\

CNN-FNO-CNN(PINO) & 2.33e-02 & 3.92e-01  & 2.35e-01 & 1.22e+01 &9.21e-01 & 9.48e+02 & 1.87e-01 & 3.69e+03 \\

CNN-FNO-CNN(TTOP) & 2.31e-02& \textbf{6.48e-02}& \underline{3.59e-02}& \textbf{9.55e-02}&6.07e-01 & \textbf{1.34e+00}&1.93e-01 &\underline{1.60e+00} \\
    \bottomrule
  \end{tabular}
  }
\end{table*}

\begin{table*}[!htbp]
  \caption{\textbf{Ablation study on architecture for Navier-Stokes}. Bold indicates the best result, and underlining indicates the second-best.}
  \label{tab:NS_AB}
  \centering
  \scriptsize
  \setlength{\tabcolsep}{3pt}
  \resizebox{\linewidth}{!}{%
  \begin{tabular}{lccccccccc}
    \toprule
    & \multicolumn{2}{c}{In-distribution} & \multicolumn{2}{c}{Viscosity-shift} & 
    \multicolumn{2}{c}{Forcing-shift} & 
    \multicolumn{2}{c}{Super-resolution}
    \\
    \cmidrule(r){2-3} \cmidrule(r){4-5} \cmidrule(r){6-7} \cmidrule(r){8-9}
    Method  & $L_2$ Rel & PDE Res & $L_2$ Rel & PDE Res & $L_2$ Rel & PDE Res& $L_2$ Rel & PDE Res\\
    \midrule
ENS(Ours) &\textbf{2.22e-03} &  \underline{4.69e-05} &  \textbf{3.05e-02} & \underline{1.12e-02} &  \textbf{5.31e-03} & \underline{7.40e-04} & \textbf{3.98e-02}  & \underline{9.11e-02}\\

VideoPDE &3.02e-02&2.38e-03&  1.22e-01 &2.42e-02 & 5.30e-01 & 4.53e-02 & 3.86e-01 & 7.96e-01\\

VideoPDE(PINO) &2.99e-02&3.91e-04& 1.31e-01  & 2.56e-03& 5.24e-01 & 2.02e-02 & 4.09e-01 & 9.91e-01\\

VideoPDE(TTOP) & \underline{8.69e-03}& \textbf{1.47e-05}& \underline{7.09e-02} & \textbf{4.97e-04} & \underline{2.42e-02}&\textbf{1.56e-04}  &\underline{7.81e-02} & \textbf{9.60e-04} \\
    \bottomrule
  \end{tabular}
  }
\end{table*}

\begin{table*}[!htbp]
  \caption{\textbf{Ablation study on architecture for Kolmogorov flow}. Bold indicates the best result, and underlining indicates the second-best.}
  \label{tab:KF_AB}
  \centering
  \scriptsize
  \setlength{\tabcolsep}{3pt}
  \resizebox{\linewidth}{!}{%
  \begin{tabular}{lccccccccc}
    \toprule
    & \multicolumn{2}{c}{In-distribution} & \multicolumn{2}{c}{Viscosity-shift} & 
    \multicolumn{2}{c}{Forcing-shift} & 
    \multicolumn{2}{c}{Super-resolution}
    \\
    \cmidrule(r){2-3} \cmidrule(r){4-5} \cmidrule(r){6-7} \cmidrule(r){8-9}
    Method  & $L_2$ Rel & PDE Res & $L_2$ Rel & PDE Res & $L_2$ Rel & PDE Res& $L_2$ Rel & PDE Res\\
    \midrule
ENS(Ours) &\textbf{8.08e-03}  & \underline{2.50e-01}& \textbf{3.42e-02}& \underline{8.05e-01} & \textbf{2.06e-02} &  \underline{3.80e-01} & \textbf{1.98e-02} &\underline{5.55e+00} \\

VideoPDE &4.12e-02&3.42e+00&  5.41e-01 &2.35e+01 & 5.04e-01 & 3.40e+01 & 5.56e-01 & 2.57e+03\\

VideoPDE(PINO) &\underline{4.08e-02}&7.73e-01&  5.33e-01 &8.36e+00 & 4.81e-01 & 1.53e+01 & \underline{5.36e-01} & 2.36e+03\\

VideoPDE(TTOP) & 4.34e-01& \textbf{1.09e-01}&  \underline{2.54e-01} &\textbf{6.54e-02} & \underline{2.14e-01} & \textbf{2.77e-02} & 7.15e-01& \textbf{1.42e+00}\\
    \bottomrule
  \end{tabular}
  }
\end{table*}

\begin{figure*}[!htbp]
    \centering
    \includegraphics[width=1.0\linewidth]{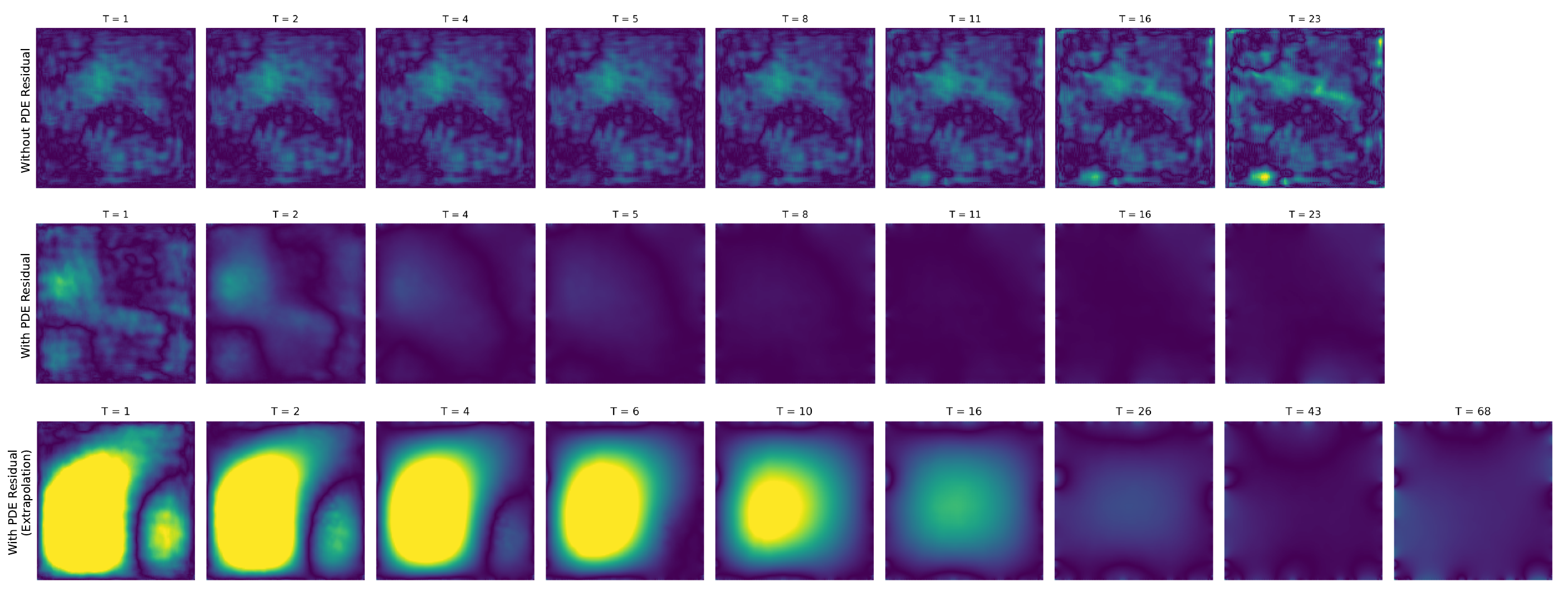}
    \caption{\textbf{Reconstruction loss ($L_1$ difference) over iterations}. From top to bottom, we visualize the $L_1$ error trajectory for the ablation baseline without PDE residual conditioning, ENS model on in-distribution data, ENS model on out-of-distribution data. Clearly, the effectiveness of ENS does not come from the increased computation only, evidenced by the failure of the ablation baseline. Note that out-of-distribution cases take more iterations to converge.}
    \label{fig: RLvis}
\end{figure*}

\begin{figure*}[!htbp]
    \centering
    \includegraphics[width=1.0\linewidth]{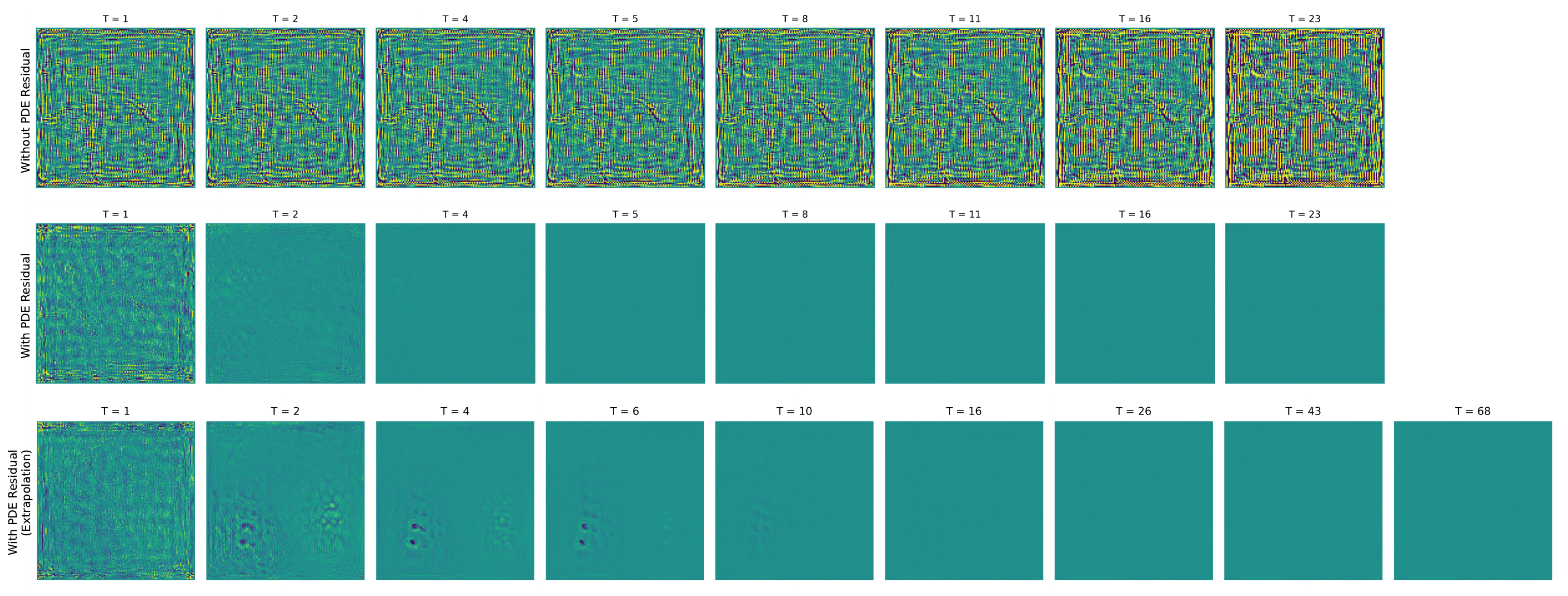}
    \caption{\textbf{PDE residual over iterations}. From top to bottom, we visualize the PDE residual trajectory for the ablation baseline without PDE residual conditioning, ENS model on in-distribution data, ENS model on out-of-distribution data. Clearly, without conditioning on the residual field, the PDE residual even increases during the correction process.  }
    \label{fig: PDEvis}
\end{figure*}

\begin{figure*}[!htbp]
    \centering
    \includegraphics[width=1.0\linewidth]{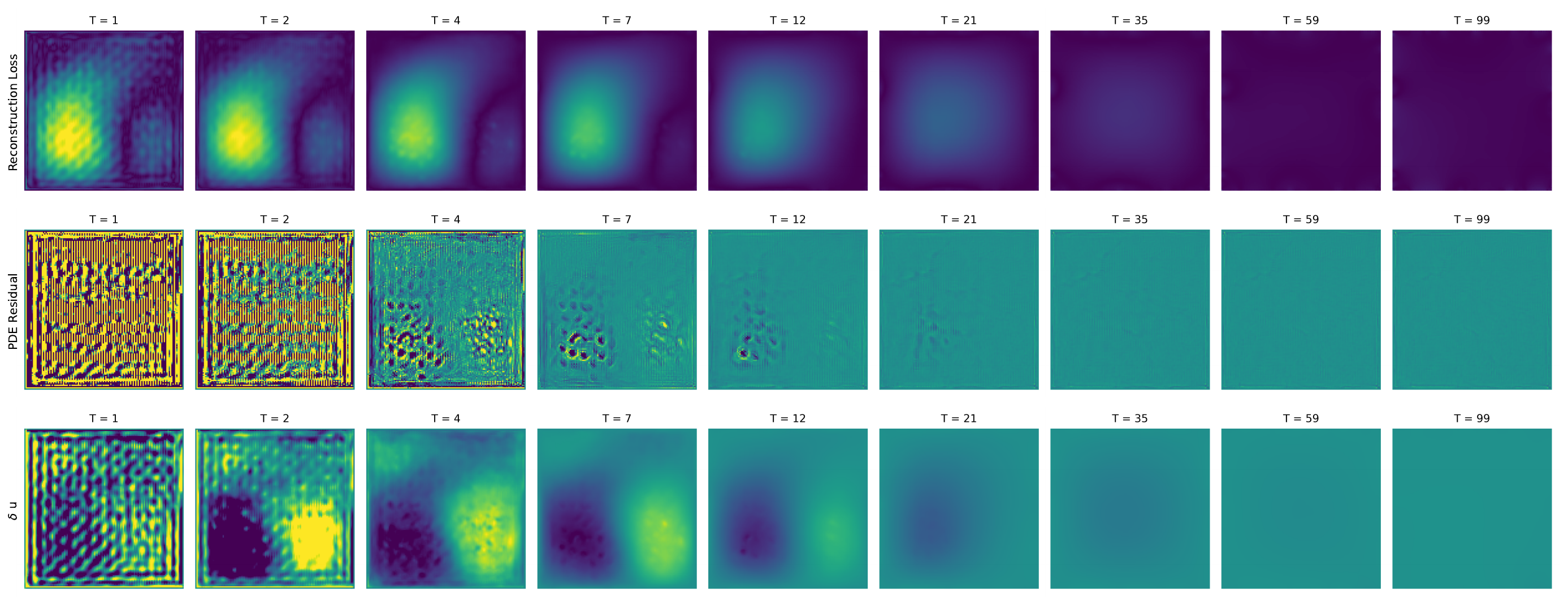}
    \caption{\textbf{Visualization of the Reconstruction loss, PDE residual, and the corresponding learned correction term $\delta\hat{u}^{(k)}$ for extrapolated source input fields}. The learned correction closely aligns with the PDE residual and progressively removes high-residual regions, leading to a monotonic reduction in both errors. Due to distribution shift, it takes  60\textasciitilde 90 iterations for ENS to converge.}
    \label{fig: delta_u}
\end{figure*}

\begin{figure*}[!htbp]
    \centering
    \includegraphics[width=1.0\linewidth]{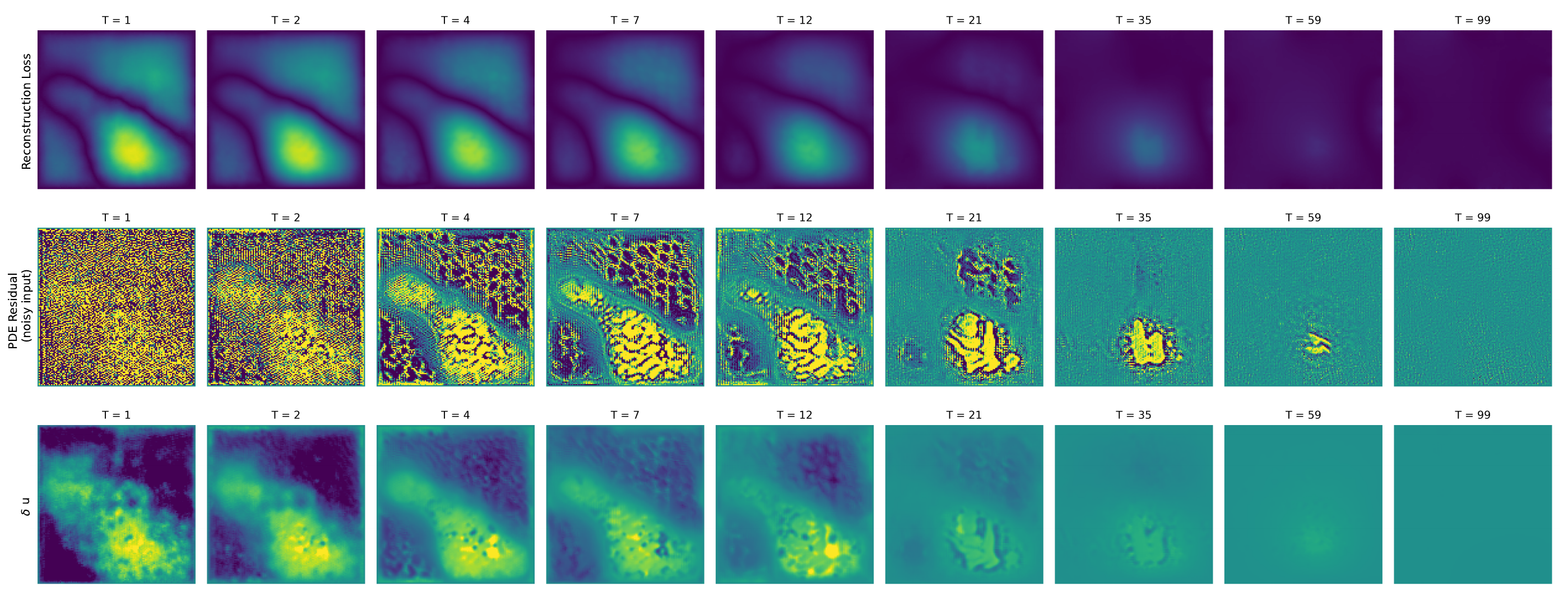}
    \caption{\textbf{Visualization of the Reconstruction loss, PDE residual, and the corresponding learned correction term $\delta\hat{u}^{(k)}$ for out-of-distribution, Gaussian-noised source input fields}. Despite the corrupted input field, the learned correction exhibits strong spatial alignment with the PDE residual, resulting in a monotonic decrease in both the PDE residual and reconstruction error.}
    \label{fig: delta_u_noisy}
\end{figure*}

\begin{figure}[!htbp]
    \centering
    \includegraphics[width=0.8\linewidth]{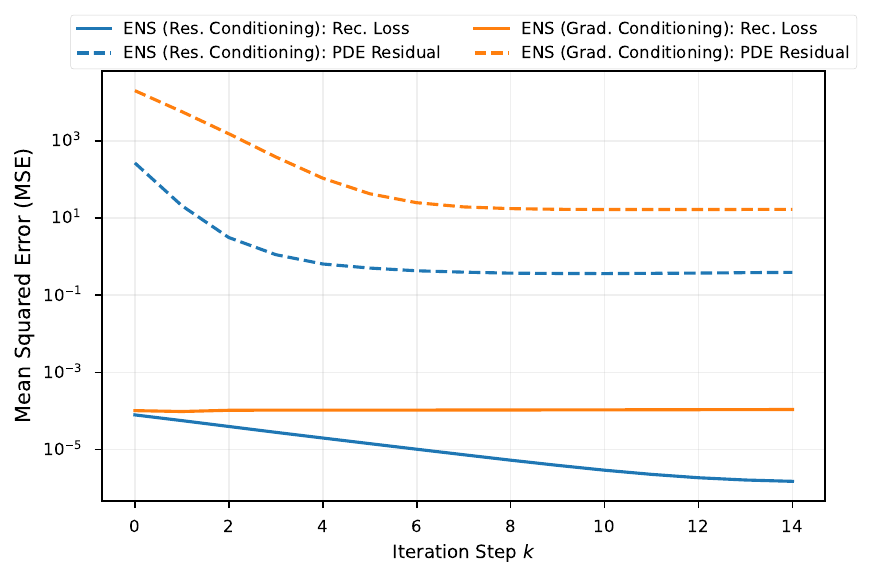}
    \caption{\textbf{Comparison of conditioning strategies for the ENS corrector}. We evaluate the optimization trajectories when conditioning on the PDE residual $\mathcal{R}(u_k)$ versus the physics-gradient $\nabla_{u_k}\mathcal{L}_{res}$. While gradient-conditioning accelerates the minimization of the PDE residual, it results in a stagnation of reconstruction loss. In contrast, residual-conditioning facilitates the simultaneous convergence of both PDE residual and reconstruction loss.}
    \label{fig: gradient}
\end{figure}

\section{Related Work (Continued)}
\label{sec: rel_appendix}
 
\paragraph{Classical numerical solvers.}
The dominant classical approaches to solving PDEs — finite element methods, finite difference methods, and spectral methods — rely on discretizing the governing equations into large linear or nonlinear systems \citep{aliabadi2020boundary, solin2005partial}. Although these methods come with strong convergence guarantees, their cost scales unfavorably with resolution and dimensionality, and iterative solvers such as multigrid and Krylov methods \citep{saad2003iterative} still require significant per-instance computation. ENS is motivated precisely by this bottleneck: it amortizes the correction process across a family of PDE instances, replacing expensive per-instance numerical optimization with a single learned recurrent policy.

\paragraph{Feed-forward Neural Operators} Neural operators learn mappings between function spaces, enabling fast prediction across entire PDE families. FNO \citep{li2020fourier} parameterizes the integral kernel in Fourier space; DeepONet \citep{lu2021learning} provides universal approximation guarantees via a branch-trunk decomposition; and a unified theoretical treatment is given in \citet{kovachki2023neural}. Foundation models such as POSEIDON \citep{herde2024poseidon} extend operator learning to large-scale pretraining over diverse PDE families. All of these methods are \emph{physics-blind} at inference time: the network receives no feedback about its own prediction error, which can lead to large constraint violations even when outputs appear visually plausible. While PINNs \citep{raissi2019physics} instead parameterize the solution of a single PDE instance, it requires expensive retraining per instance (Fig.~\ref{fig:motivation_convergence}).
 
\paragraph{External residual correction.} To improve the physical accuracy of feed-forward models, a growing body of work incorporates the PDE residual as a correction signal at inference time, applied externally to learned network. Optimization-based methods include PINO \citep{li2024physics}, which performs instance-wise fine-tuning of network parameters by minimizing $\|r(\hat{u})\|^2$; DiffusionPDE \citep{huang2024diffusionpde} and CoCoGen \citep{jacobsen2025cocogen}, which steer the reverse diffusion process with residual gradients; PDM \citep{christopher2024constrained}, which projects each diffusion step onto the feasible constraint set; PCFM \citep{utkarsh2025physics}, which applies penalized Gauss-Newton corrections at each flow-matching step; and ECI sampling \citep{cheng2025gradient}, which adapts pre-trained flow-matching models to satisfy hard constraints in zero-shot manner. Analytical approaches instead construct a closed-form correction from the residual directly: \citet{cao2023residual} and \citet{jha2024residual} solve a linear variational problem at the neural operator's prediction, yielding a one-shot Newton-like correction; PhysicsCorrect \citep{huang2026physicscorrect} similarly applies a linearized least-squares solve using a precomputed Jacobian pseudoinverse. Despite their differences, none of these methods feeds the residual field as a direct input channel to a learned corrector network.

\paragraph{Learned iterative solvers and algorithm unrolling.}
Iterative correction has also been explored without explicit residual feedback. PDE-Refiner \citep{lippe2023pde} applies diffusion-inspired multi-step refinement using noise as the correction signal across refinement levels, enabling accurate modeling of all frequency components in long-horizon rollouts. INC \citep{wei2026inc} embeds learned corrections within the governing equations of a coarse numerical solver, provably reducing autoregressive error amplification in chaotic regimes. Both demonstrate the value of iterative learned correction, but neither uses the PDE residual field as an input conditioning.

Learning to optimize \citep{andrychowicz2016learning}, Recurrent Inference Machines (RIMs)~\citep{putzky2017recurrent}, and Deep Equilibrium Models \citep{bai2019deep} demonstrate that iterative correction procedures can be learned rather than hand-designed. In particular, RIMs learn recurrent updates conditioned on the current estimate and likelihood gradients. ENS shares this iterative spirit but differs in a key respect: the corrector is explicitly conditioned on the \emph{spatial structure} of the PDE residual field, giving the network direct access to where and how its prediction violates the governing equations rather than learning a general-purpose update rule.

\paragraph{Diffusion and generative models for scientific problems.}
Diffusion models \citep{ho2020denoising, song2020score, karras2022elucidating} have driven a wave of applications across the natural sciences: RFdiffusion \citep{watson2023novo} and DiffSBDD \citep{schneuing2024structure} enable de~novo protein and drug design; \citep{hoogeboom2022equivariant} generate 3D molecular conformations with equivariant architectures; CDVAE \citep{xie2021crystal} generates stable crystal structures via score-matching; and GenCast \citep{price2023gencast} produces probabilistic weather forecasts that surpass operational numerical prediction systems. DiffusionPDE \citep{huang2024diffusionpde} applies generative priors to PDE
solving under partial observation, using physical residuals as external guidance,
in contrast to ENS, which conditions a learned corrector directly on the residual field.


\end{document}